\providecommand{\AnonymousSubmission}{0}
\documentclass{article}

\newif\ifanonymous
\if\AnonymousSubmission1 \anonymoustrue\fi
\ifanonymous
  \usepackage{tmlr}            
\else
  \usepackage[preprint]{tmlr}  
\fi

\usepackage[utf8]{inputenc}
\usepackage[T1]{fontenc}
\usepackage{hyperref}
\usepackage{url}
\usepackage{booktabs}
\usepackage{amsfonts,amsmath,amssymb,amsthm}
\usepackage{nicefrac}
\usepackage{microtype}
\usepackage{xcolor}
\usepackage{graphicx}
\usepackage{multirow}
\usepackage{tabularx}
\usepackage{enumitem}
\usepackage{placeins}
\usepackage{needspace}
\usepackage[font=small,labelfont=bf]{caption}
\usepackage[capitalise,noabbrev]{cleveref}

\hypersetup{pdftitle={Input-Blind Controls Produce Substantial Oracle Headroom for Layer Programs in Multiple-Choice Evaluation},pdfsubject={Placebo comparisons of layer-program oracle headroom with three independent random-direction draws per model},colorlinks=true,linkcolor=[rgb]{0.1,0.25,0.55},citecolor=[rgb]{0.1,0.25,0.55},urlcolor=[rgb]{0.1,0.25,0.55}}
\setlist{nosep,leftmargin=*}

\input{generated/numbers.sty}
\input{generated/redraw.sty}
\RequirePackage{fix-cm}
\RequirePackage{array}
\RequirePackage{textcomp}
\definecolor{muted}{gray}{0.22}
\newcommand{\ci}[3]{\ensuremath{#1\,\textcolor{muted}{[#2,#3]}}}
\newcommand{\cib}[3]{\ensuremath{\mathbf{#1}\,\textcolor{muted}{[#2,#3]}}}
\newcommand{\num}[1]{\ensuremath{#1}}
\newcommand{\numb}[1]{\ensuremath{\mathbf{#1}}}
\newcommand{\inlineci}[3]{\ci{#1}{#2}{#3}}
\newcommand{\inlinecib}[3]{\cib{#1}{#2}{#3}}

\DeclareMathVersion{approvedtable}
\SetSymbolFont{operators}{approvedtable}{OT1}{ztmcm}{m}{n}
\SetSymbolFont{letters}{approvedtable}{OML}{ztmcm}{m}{it}
\SetSymbolFont{symbols}{approvedtable}{OMS}{ztmcm}{m}{n}
\SetSymbolFont{largesymbols}{approvedtable}{OMX}{ztmcm}{m}{n}
\SetMathAlphabet{\mathbf}{approvedtable}{OT1}{ptm}{bx}{n}
\SetMathAlphabet{\mathit}{approvedtable}{OT1}{ptm}{m}{it}
\newcommand{\approvedtablefont}{\fontfamily{ptm}\selectfont\mathversion{approvedtable}%
\thinmuskip=2mu\medmuskip=2.5mu plus 1mu minus 1mu\thickmuskip=4mu plus 1.5mu minus 1mu}
\newcommand{\val}[1]{\ifcsname val@#1\endcsname\csname val@#1\endcsname\else\errmessage{Missing generated value #1}\fi}
\newcommand{\storedci}[1]{\val{#1}\,\textcolor{muted}{\val{#1.ci}}}
\newcommand{\storedcib}[1]{\textbf{\val{#1}}\,\textcolor{muted}{\val{#1.ci}}}

\newcommand{\Grepro}{G_{\mathrm{repro}}}
\newcommand{\Graw}{G_{\mathrm{raw}}}
\newcommand{\RD}{\textsf{RD}}
\newcommand{\RDf}{\textsf{RD-fixed}}
\newcommand{\RDi}{\textsf{RD-item}}
\newcommand{\SH}{\textsf{SH}}
\newcommand{\Rot}{\textsf{Rot}}
\newcommand{\RDm}{\text{\textsf{RD-fixed}}}
\newcommand{\Real}{\mathrm{Real}}
\newcommand{\Acc}{\mathrm{Acc}}

\newcommand{\Tres}{T_{\mathrm{rescue}}}

\newtheorem{proposition}{Proposition}
\newtheorem{observation}{Observation}
\crefname{observation}{Observation}{Observations}
\Crefname{observation}{Observation}{Observations}
\theoremstyle{definition}

\title{Input-Blind Controls Produce Substantial Oracle Headroom\\ for Layer Programs in Multiple-Choice Evaluation}

\ifanonymous
  \author{}
\else
  \author{\name Yibei Guo \email guo.yibeii@gmail.com\\
    \addr College of Aeronautics and Engineering, Kent State University
    \AND
    \name Rui Liu {\normalfont\small(corresponding author)} \email ruiliu.robotics@gmail.com\\
    \addr College of Aeronautics and Engineering, Kent State University}
  \hypersetup{pdfauthor={Yibei Guo, Rui Liu}}
\fi


\begin{document}
\raggedbottom

\maketitle

\begin{abstract}
Adaptive computation aims to improve language-model inference by tailoring execution to each input.
For layer programs, oracle evaluations use known answers to estimate the potential gain from this flexibility, before a practical selector is available.
However, a gain from selection does not by itself explain why the chosen programs help.
This study investigates this distinction using 32 layer-skipping and repetition programs on two models and \val{L.items} multiple-choice items.
The analysis compares their gains over a fixed action selected without the evaluation prompt with those of input-blind perturbations at the same sites, re-evaluating selections on another prompt.
With shared option order, the controls give \val{redraw@qwen.G_rd.range} and \val{redraw@llama.G_rd.range} percentage points of headroom on Qwen3-4B-Base and Llama-3.1-8B, exceeding the real programs' \val{redraw@qwen.d1.G_real} and \val{redraw@llama.d1.G_real} in all three random-direction draws per model.
They match answer-change rate only, and the ordering depends on the menu: in post hoc comparisons, real programs lead on Llama's repeat-only menu in every draw.
A smaller KL-calibrated comparison, including an input-dependent control, favours real programs in point estimate, with inconclusive corrected tests.
Fixed letter offsets produce headroom of similar scale.
Rotating options sharply reduces both families' headroom, while leaving positive real-minus-control differences of \val{redraw@qwen.D_perm.range} and \val{redraw@llama.D_perm.range} points; their magnitudes and statistical support depend on further adjustments and the reference.
A supplementary generated-answer test finds that search-selected programs keep a \val{T.T}-point advantage over programs selected for other problems after rewording, without a placebo comparison.
These results show that substantial headroom can persist across prompts with shared option order without establishing a benefit specific to the selected layer computation; neither ordering against these controls identifies that benefit.
\end{abstract}

\section{Introduction}
\label{sec:intro}

Different inputs may benefit from different ways of executing the same language model.
A growing line of work explores this possibility by skipping, repeating or looping a frozen model's layers for each input \citep{li2025cola,li2026polar,heakl2026drllm}.
Oracle evaluations help assess the opportunity for adaptation before a practical selector is available: they use known answers to choose a compute action for each input.
Gains over the unmodified forward pass show that the candidate actions offer useful choices.
Some of these works also interpret the gains: CoLa sees in them ``the shortcomings of using a fixed architecture of pre-trained LLMs for inference on different samples'' \citep{li2025cola}, and PoLar writes that ``higher latent execution complexity is not merely helpful, but often necessary for solving harder inputs'' \citep{li2026polar}.

Such interpretations require distinguishing the value of selection from the contribution of the selected computation.
An oracle uses the answer to choose among interventions that can improve some predictions and worsen others.
A menu can therefore offer useful choices even when its interventions omit the corresponding layer edits.
Re-measuring the chosen actions on a second prompt---the same question after different few-shot examples---tests whether their advantages persist.
An item's stable sensitivity to perturbation and structure shared by the two prompts can still support a persistent gain.
This motivates the central question: \emph{what do oracle gains establish about the contribution of the intended computation?}

The experiments use controlled menus of single-segment skipping and repetition programs in multiple-choice evaluation.
To distinguish the value of choosing different actions across inputs from the benefit of one generally useful action, this study measures the oracle's gain over the best fixed action selected without the evaluation prompt (on separate development items in the main study).
This gain is the oracle’s \emph{headroom}.
Layer programs are then compared with placebo actions, in the spirit of randomization tests for saliency maps \citep{adebayo2018sanity}.
These controls perturb the model at the same sites, are calibrated on separate items using a named effect statistic and tolerance, and offer the oracle the same number of candidates.
To measure headroom obtainable without the corresponding input-dependent layer edit, the study uses an input-blind injected increment (\cref{fig:overview}).
Finally, rotating the options of the second prompt moves the correct answer to another letter, testing how much the gain depends on the option order shared by the prompts.

\begin{figure}[t]
\centering
\includegraphics[width=\textwidth]{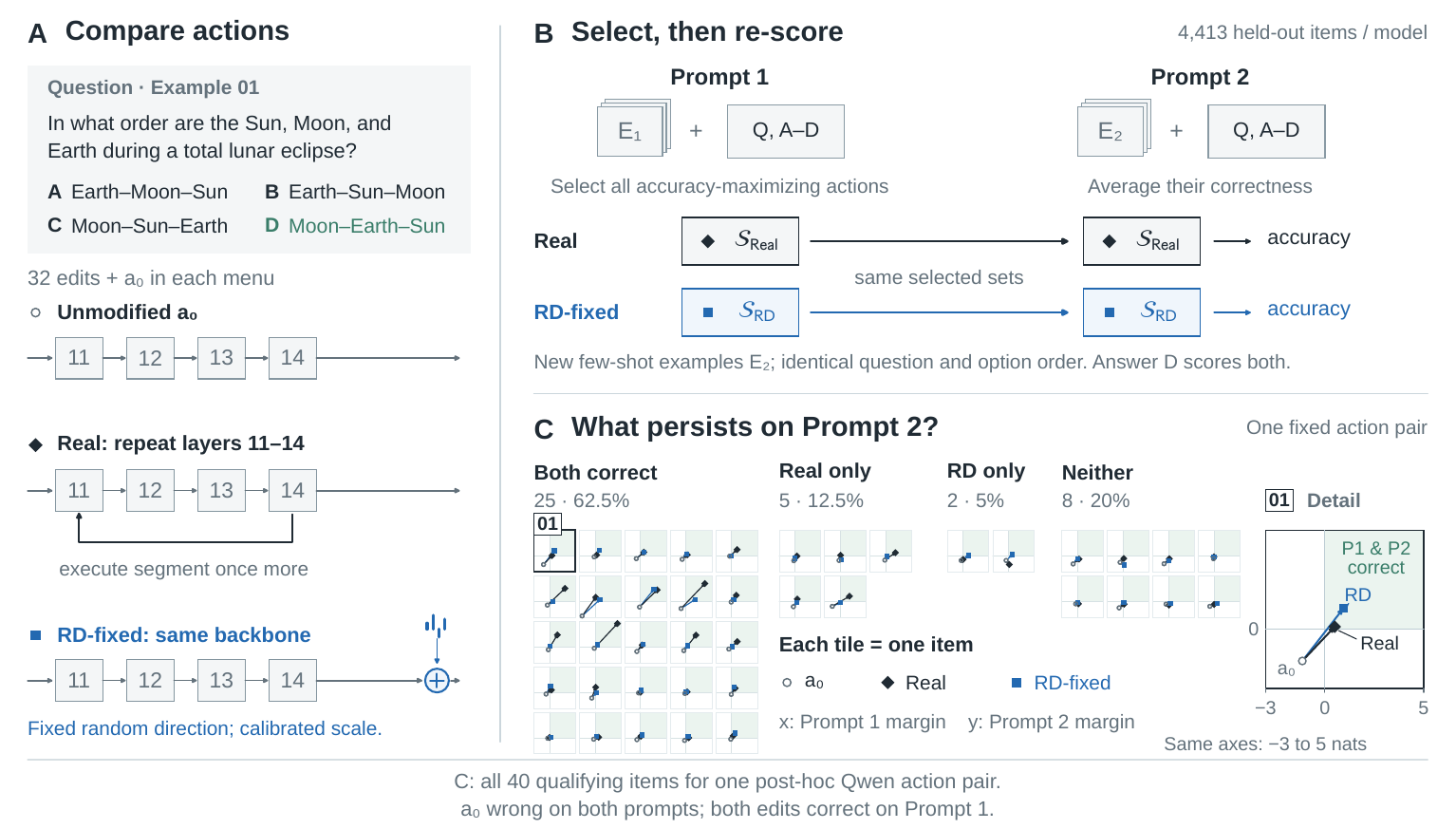}
\caption{\textbf{Cross-prompt oracle comparison} (Qwen3-4B-Base; original RD draw).
(A)~Each family offers 32 edits and the unmodified pass $a_0$.
The illustrated real action repeats layers 11--14; RD-fixed preserves the unmodified execution and adds an input-independent update at the segment exit.
Its direction, norm profile and scale are shared across items and prompts; calibration matches answer-change rate only, on separate data.
(B)~For each item, all accuracy-maximizing actions on Prompt 1 are retained and their correctness averaged on Prompt 2.
Only the few-shot examples change; the question and ordered options stay the same.
Target answer labels are used for scoring and are withheld from the model.
The full study uses 4,413 items per model and averages both selection directions.
(C)~All 40 qualifying items for one post-hoc selected action pair: $a_0$ is wrong on both prompts and both edits are correct on Prompt 1.
The counts describe this conditional subset, not the full oracle menu.
Coordinates are correct-option log-probability margins against the strongest incorrect option, on common linear $-3$ to $5$ axes; positive values indicate correctness.
Arrows connect unmodified and edited outputs, not hidden states. Detail 01 enlarges the printed question's tile without adding an observation (\cref{app:figure-notes}).}
\label{fig:overview}
\end{figure}

\Needspace{8\baselineskip}
The experiments yield three main findings:
\begin{enumerate}[itemsep=3pt]
\item \textbf{Input-blind controls produce substantial cross-prompt headroom.}
With shared option order, \RDf{} exceeds real-program headroom on the full menus in all three random-direction draws for both models.
These controls omit the corresponding layer edits and are calibrated to answer-change rate only.
Their gains demonstrate that substantial headroom can persist across prompts without establishing a benefit specific to the selected layer computation (\cref{sec:large,sec:redraw}).

\item \textbf{The ordering depends on the menu and control construction.}
On Llama's repeat-only menu, real programs lead in every draw in post hoc comparisons.
In a smaller KL-calibrated study using additional control constructions, controls reach \val{S.rcross.min}--\val{S.rcross.max}\% of real headroom in point estimate, with inconclusive corrected comparisons.
Thus, the full-menu ordering does not extend uniformly across the tested settings (\cref{sec:large,sec:redraw}).

\item \textbf{Shared option order strongly affects the gains.}
Fixed letter offsets produce headroom of similar scale in the original-draw diagnostics.
Rotating the options sharply reduces both families' headroom, while leaving a positive real-minus-control difference in every full-menu draw.
These residual differences depend on the comparison design and do not identify a computation-specific benefit (\cref{sec:mechanism}).
\end{enumerate}

A supplementary generated-answer experiment measures the advantage retained by search-selected programs after rewording, without a placebo comparison (\cref{sec:transfer}).
The placebo findings concern the controlled menus studied here; they do not constitute tests of CoLa or PoLar's published search systems.

\Needspace{7\baselineskip}
\section{What oracle headroom can identify}
\label{sec:theory}

\paragraph{Setup.}
Items are multiple-choice questions scored by the next-token log-probabilities of the option letters after ``Answer:''.
For item $i$, prompt $v\in\{1,2\}$ and action $a$, the margin $U_i^v(a)$ is the log-probability of the correct letter minus the largest log-probability of an incorrect letter, and $\Acc_i^v(a)=\mathbf{1}[U_i^v(a)>0]$.
An action family is $F=\{a_0,a_1,\dots,a_K\}$, where $a_0$ is the unmodified forward pass.
In the fixed-order design, each item is asked with two prompts that differ only in their few-shot examples; the question, the options and their order are identical, so the correct letter is the same in both (\cref{fig:overview}).

\paragraph{Oracle quantities.}
With 0/1 accuracy, many items have several best actions. Averaging over them avoids choosing an arbitrary winner or using extra margin information.
Preferring the unmodified action on ties answers a different selection question; results under this rule are reported as a sensitivity check. The primary averaging rule was inherited unchanged by the later rotated design.
Let $A_i^{s}=\arg\max_a \Acc_i^{s}(a)$ be the set of actions tied for best on the selection prompt $s$, and let $a_S$ be a static action, the single action applied to every item, chosen without the evaluation prompt.
The cross-prompt and same-prompt headroom of $F$ are
\begin{align}
\Grepro(F)&=\tfrac12\sum_{(s,s')}\operatorname{mean}_i\Bigl[\operatorname{mean}_{a\in A_i^{s}}\Acc_i^{s'}(a)-\Acc_i^{s'}(a_S)\Bigr], \label{eq:headroom}\\
\Graw(F)&=\tfrac12\sum_{v}\operatorname{mean}_i\Bigl[\max_a\Acc_i^{v}(a)-\Acc_i^{v}(a_S)\Bigr], \nonumber
\end{align}
where $(s,s')$ runs over both orderings of the two prompts: the cross-prompt oracle chooses actions on one prompt and scores the same actions on the other, in both directions, whereas same-prompt headroom chooses and scores on the same prompt.
A placebo family $N$ is compared with the real family through $\rho(N)=\Grepro(N)/\Grepro(\Real)$ and $\Delta(N)=\Grepro(\Real)-\Grepro(N)$.
The analysis never takes a maximum over placebo families, which would itself be a hindsight selection.

\paragraph{Identifying computation-specific effects.}
To distinguish possible sources of an observed effect, write the effect of action $k$ on item $i$ in prompt $v$ as
\begin{equation}
x_{ik}^{v}=U_i^v(a_k)-U_i^v(a_0)=\mu_k+g_{ik}+h_{ik}+e_{ik}^{v},
\label{eq:model}
\end{equation}
with mean effect $\mu_k$, $g_{ik}$ a stable item-by-action effect that a \emph{non-specific} perturbation of the same kind would also produce, $h_{ik}$ a stable effect \emph{specific} to the intended computation, and noise $e_{ik}^{v}$.
The random terms are centered over items, so $h$ excludes mean effects: any mean computation-specific contribution is part of $\mu_k$.
The common mean and prompt-invariant stable effects are assumptions for the fixed-order comparison.
For a layer program, the intended computation is the input-dependent update that the skipped or repeated layers write at the segment exit; since $g$ is defined by a class of non-specific perturbations, $h$ is defined only relative to that class, which a placebo makes explicit.

\begin{observation}[Non-identifiability under the additive decomposition]
\label{prop:nonid}
Under \cref{eq:model}, however many prompts, seeds or resamples of the real family are observed, their joint distribution depends on $g$ and $h$ only through $g+h$.
No statistic of the real family's outcomes, including same- and cross-prompt headroom and resampling estimates, can separate the specific from the non-specific part of the stable effect.
\end{observation}

Replacing $(g,h)$ by $(g+h,0)$ or by $(0,g+h)$ leaves the distribution of every observable unchanged (\cref{app:theory}).
This observation frames an identification limit of the decomposition without further restrictions on $g$ and $h$ (\cref{fig:mechanism}); it is not an empirical finding about their relative sizes.
Re-measurement can probe variation across prompts, but an effect shared by those prompts remains compatible with either source.
Changing shared structure probes a different property: an action that raises one answer letter can help on both prompts when the correct option keeps that letter, yet hurt after that option moves.
Interpreting an effect as computation-specific therefore requires additional assumptions or interventions beyond repeated observations of the real family.

\begin{figure}[t]
\centering
\includegraphics[width=0.92\textwidth]{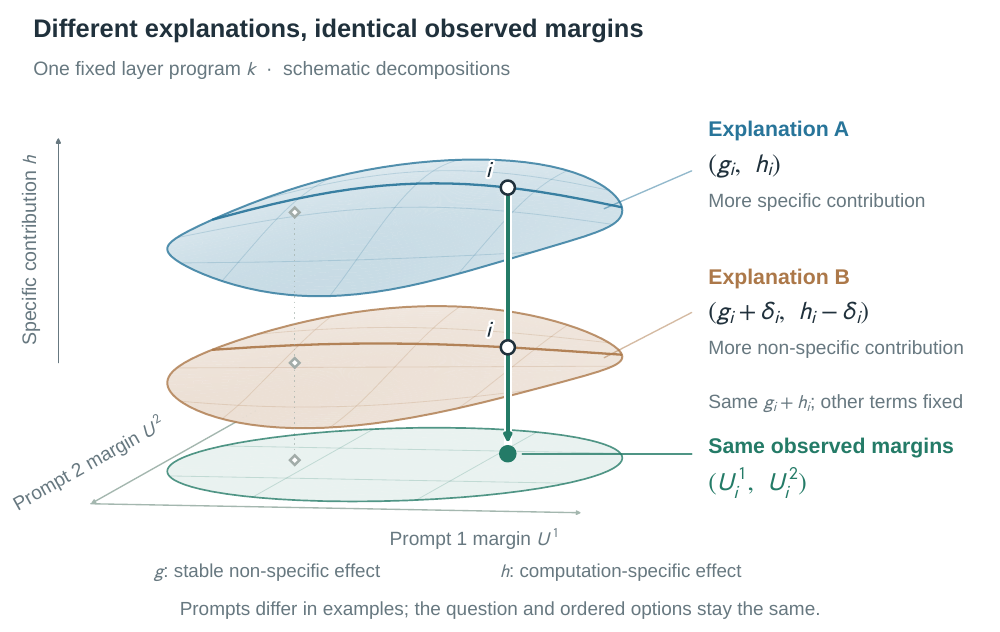}
\caption{\textbf{Different explanations, identical observed margins.}
A and B illustrate alternative decompositions of one layer program's effect.
The white circles track question $i$: replacing $(g_i,h_i)$ by $(g_i+\delta_i,h_i-\delta_i)$, with $\delta_i>0$, preserves $g_i+h_i$.
With all other terms fixed, both points project to the same observed margin pair $(U_i^1,U_i^2)$, shown in green.
Height represents the computation-specific contribution $h$, defined relative to the chosen class of non-specific perturbations.
The surfaces are schematic, not fitted functions or successive computations; A/B are not real/placebo conditions.
The prompts change the few-shot examples while preserving the question and ordered options.
Equality concerns margins, not full token distributions.}
\label{fig:mechanism}
\end{figure}

\paragraph{What a placebo can show.}
A placebo with an input-blind injected increment omits the corresponding input-dependent edit, giving $h\equiv0$ for that increment by construction.
The unmodified state and downstream model still depend on the input, and the oracle still selects actions using answers.
Any resulting headroom establishes what this control can produce; that existence claim needs no distributional match to the real family.
Attributing part of the real family's headroom to a particular source is a stronger task.
It would require assumptions about the control's stable and noisy effects, and the nonlinear oracle operation prevents a placebo-to-real ratio from being read directly as a fraction of the real headroom (\cref{app:theory}).
Therefore, $\rho$ and $\Delta$ serve as comparisons with named controls.

\section{Placebo comparisons}
\label{sec:method}
\paragraph{Actions and injection site.}
To compare real and placebo interventions at the same injection sites and with the same choice budget, the study pairs each layer program with one placebo action.
A layer program skips a contiguous segment of layers or repeats it once, and then runs the remaining layers as usual.
Its effect enters the residual stream at one well-defined point, the exit of the segment: with $z_t$ the unmodified state there at token position $t$, the program replaces it by $z_t+\delta_{k,t}$.
A placebo action replaces $z_t$ by $z_t+\tilde\delta_{k,t}$ at the same point and runs the same downstream layers.
Both menus also include the unmodified action $a_0$.

\paragraph{The main control: a fixed random direction.}
To measure how much oracle headroom can arise from input-blind updates, the main study uses \RDf{}, whose injected increment reads neither the current input state nor its answer.
Its injected increment is
\[
\tilde\delta_{k,t}=c_k\,r_k(t)\,\hat u_{k,t}.
\]
Here $r_k(t)$ is a norm profile of the real update, estimated on separate calibration items as a function of distance from the sequence end, and $c_k$ is a program-level scale.
Each layer receives a fixed random basis with one row per distance from the sequence end; a program's unit direction $\hat u_{k,t}$ is the normalized signed sum of these rows over the layers it edits ($-1$ for skipping, $+1$ for repetition).
This preserves correlations between overlapping programs and gives opposite directions to the skip and repeat of the same segment.
The direction, norm profile and scale are fixed across evaluation items and prompts.
The injected increment is identical on the text shared by an item's two prompts because positions are aligned from the sequence end.

\paragraph{Calibration and held-out checks.}
To align how often real and placebo actions change the predicted answer, the protocol calibrates placebo scales on separate items and checks the resulting answer-change rates on held-out evaluation items.
Calibration uses no answer labels.
The \emph{answer-change rate} (churn) is the fraction of item--prompt pairs whose predicted letter differs from the unmodified pass.
This statistic records how often an action can alter the accuracy outcome.
On evaluation items, the same label-free check requires at least 30 of 32 programs within tolerance, all 32 within a broad band, and a median churn ratio in $[0.9,1.1]$ (\cref{app:placebos}).
A family that fails is excluded whole from the outcome comparison and reported as a failed check; dropping individual actions would change the choice budget.
Passing this rule checks one effect statistic, rather than the joint distribution of stable and prompt-varying effects.
Independent redraws repeat calibration and the held-out checks (\cref{sec:redraw,app:redraw}).

\begin{table}[t]
\centering
\caption{Placebo definitions.}
\label{tab:placebos}
\begin{minipage}{396pt}\approvedtablefont
\begingroup\fontsize{8.5}{10}\selectfont
\setlength{\tabcolsep}{3pt}\renewcommand{\arraystretch}{1.04}
\begin{tabular*}{\linewidth}{@{\extracolsep{\fill}}>{\raggedright\arraybackslash}p{157pt}cccc@{}}
\toprule
\textbf{Property} & \textbf{\RDf} & \textbf{\RDi} & \textbf{\SH} & \textbf{\Rot} \\
\midrule
Study menu & 32 programs & 3 programs & 3 programs & 3 programs \\
Effect-size match & answer changes & KL & KL & KL \\
\midrule[0.3pt]
Direction shared by all items & yes & no & no & no \\
Same perturbation on shared text & yes & yes & yes & no \\
Model-generated update & no & no & other input's & own, rotated \\
Reads current input & no & no & no & yes \\
\bottomrule
\end{tabular*}\endgroup

\par\vspace{5pt}\fontsize{8.5}{10}\selectfont\raggedright
Shared text is common to both prompts of an item. For \Rot{}, the rotation matrix is shared while the input's own update varies across inputs and prompts. For \SH{} and \Rot{}, the absence of computation-specific effects remains an assumption.
\end{minipage}
\end{table}

\paragraph{Additional control constructions.}
To examine headroom under additional control constructions, a smaller three-program study evaluates item-specific random directions, updates borrowed from other inputs, and rotated real updates (\cref{tab:placebos,app:placebos}).
\RDi{} draws a random direction separately for each item and action, keeping it shared across that item's prompts.
\SH{} injects the real action's update from another, length-matched input.
\Rot{} injects the current input's update after a shared Haar-random rotation, $\tilde\delta_{k,t}=c_k Q\delta_{k,t}$.
For \SH{} and \Rot{}, absence of computation-specific effects is an assumption: one carries a model-generated update and the other reads the current input.
This study calibrates and checks median output KL divergence from the unmodified pass.
It provides comparisons under additional constructions; it does not isolate the effect of changing the calibration statistic.

\paragraph{Rotated options.}
To assess whether selected gains depend on shared option order, the second design applies a nonzero cyclic rotation to the options in prompt 2.
For an item with $n>1$ options the shift $r_i\in\{1,\dots,n-1\}$ is derived from the item's identifier, so that every option, and with it the correct answer, moves to another letter.
The few-shot examples, prompt 1, the placebo scales and the static actions are unchanged, and the matching rule is applied again.

\paragraph{Inference.}
To quantify uncertainty in real--placebo comparisons while preserving their pairing, the analysis uses bootstrap resampling with the same sampled items for both families.
In the large-menu study, evaluation items were never used for calibration, development or the choice of static actions; in the three-program study the static action is chosen on the selection prompt (\cref{app:ladder}).
Tasks are weighted equally; primary contrasts use the tie-average oracle and the static actions chosen on development items, and the other variants are sensitivity analyses.
The half-headroom threshold, tie rule, static-action rule and churn calibration preceded development outcomes. The original study stopped when its joint power requirement, including equivalence, could not be met; after inspection of development results, a separate half-headroom confirmation was specified on untouched evaluation items (\cref{app:analysis-timing}).
Option rotation, item-decoupled comparisons and the new paired intervals and tie-type decomposition were later extensions; the two fresh control draws and the separate transfer confirmation were each specified before their own new evaluations.
Intervals are 95\% percentile intervals from 10,000 paired item bootstrap replicates (stratified by task in the large-menu study); with three placebo comparisons, Holm-corrected results are also reported.
A sensitivity analysis resamples MMLU-Pro categories and BBH tasks as clusters (\cref{app:large,app:ladder}).
All intervals condition on the realized placebo draws, fitted scales, prompts and any development-fitted quantities, including static actions and letter offsets.
Bold table entries mark larger displayed point estimates under the comparison specified in each table, not statistical significance; displayed ties are unmarked.

\FloatBarrier
\section{Experiments}
\label{sec:experiments}

\subsection{Setup}
\label{sec:setup}

The experiments evaluate frozen Qwen3-4B-Base \citep{yang2025qwen3} and Llama-3.1-8B \citep{grattafiori2024llama3}, in bf16 with the language-model head in fp32.
The main menu contains 32 single-segment programs: skip or repeat once, segments of 1--4 layers, placed at the centres of four equal depth bands.
The main study compares this menu with \RDf{} on MMLU-Pro \citep{wang2024mmlupro}, the letter-choice tasks of BBH \citep{suzgun2023bbh} and ARC-Challenge \citep{clark2018arc}.
Placebo scales are calibrated on 1,000 label-free items; static actions are chosen per model, task and family on 1,000 separate development items.
The evaluation set contains \val{L.nd} items per task, \val{L.items} in total, shared by both models, with equal task weights.
Both menus are evaluated with shared option order and with prompt 2's options rotated (\cref{sec:method}).

Besides reporting both families' headroom, the analysis tests the specified threshold of half the real headroom:
$C=\Grepro(\RDm)-\tfrac12\Grepro(\Real)$, one-sided at level 0.025 on each model.
Reaching this threshold establishes a substantial effect relative to the real-program comparison; it is not a decomposition of that effect.
The detailed letter-offset and tie diagnostics below use the original RD draw; submenu results cover all three draws.
\Cref{sec:redraw} reports the independent redraws.

\subsection{Headroom with shared option order}
\label{sec:large}

\begin{table}[!htbp]
\centering
\caption{Large-menu study: headroom and design sensitivity (original RD draw).}
\label{tab:large}
\begin{minipage}{396pt}\approvedtablefont
\begingroup\fontsize{8}{9.5}\selectfont
\setlength{\tabcolsep}{1.5pt}\renewcommand{\arraystretch}{1.03}
\begin{tabular*}{\linewidth}{@{\extracolsep{\fill}}lcc@{\hspace{5pt}}cc@{\hspace{5pt}}cc@{}}
\toprule
 & \multicolumn{2}{c}{Same prompt $\Graw$} & \multicolumn{2}{c}{\textbf{\textit{G}\textsubscript{repro}: tie average}} & \multicolumn{2}{c}{$\Grepro$: $a_0$ kept} \\
\cmidrule(lr){2-3}\cmidrule(lr){4-5}\cmidrule(lr){6-7}
Family & Fixed & Rotated & Fixed & Rotated & Fixed & Rotated \\
\midrule
\multicolumn{7}{@{}l}{\textbf{Qwen3-4B-Base}} \\[1pt]
Real & \inlineci{29.1}{28.0}{30.3} & \inlinecib{29.3}{28.3}{30.3} & \inlineci{9.0}{8.1}{9.9} & \inlinecib{-3.3}{-4.0}{-2.7} & \inlineci{11.4}{10.5}{12.3} & \inlinecib{1.3}{0.7}{2.0} \\
\RDf & \inlinecib{29.2}{28.1}{30.4} & \inlineci{29.0}{28.0}{30.0} & \inlinecib{11.8}{10.9}{12.8} & \inlineci{-5.7}{-6.3}{-5.1} & \inlinecib{13.9}{13.0}{14.8} & \inlineci{-0.7}{-1.3}{-0.1} \\
$\Delta$ & \inlineci{-0.1}{-1.0}{0.8} & \inlineci{0.3}{-0.5}{1.1} & \inlineci{-2.8}{-3.8}{-1.9} & \inlineci{2.3}{1.6}{3.0} & \inlineci{-2.5}{-3.5}{-1.6} & \inlineci{2.0}{1.3}{2.7} \\
\addlinespace[4pt]
\multicolumn{7}{@{}l}{\textbf{Llama-3.1-8B}} \\[1pt]
Real & \inlineci{35.5}{34.3}{36.6} & \inlinecib{36.4}{35.4}{37.5} & \inlineci{10.1}{9.1}{11.0} & \inlinecib{-2.5}{-3.2}{-1.8} & \inlineci{11.2}{10.2}{12.1} & \inlinecib{0.4}{-0.4}{1.2} \\
\RDf & \inlineci{35.5}{34.3}{36.7} & \inlineci{36.0}{34.9}{37.1} & \inlinecib{15.6}{14.5}{16.6} & \inlineci{-6.1}{-6.7}{-5.5} & \inlinecib{16.2}{15.2}{17.2} & \inlineci{-2.7}{-3.3}{-2.0} \\
$\Delta$ & \inlineci{0.0}{-1.1}{1.1} & \inlineci{0.4}{-0.5}{1.4} & \inlineci{-5.5}{-6.6}{-4.4} & \inlineci{3.7}{2.8}{4.5} & \inlineci{-5.0}{-6.1}{-3.8} & \inlineci{3.1}{2.3}{4.0} \\
\bottomrule
\end{tabular*}\endgroup

\par\vspace{5pt}\fontsize{8.5}{10}\selectfont\raggedright
Accuracy headroom over each family's development-selected static action (per model and task), in percentage points [95\% paired, task-stratified item-bootstrap CI]. $\Graw$ selects and scores on the same prompt; $\Grepro$ selects on one and scores on the other, averaging directions. Tie average averages uniformly over best-accuracy actions; $a_0$ kept retains the unmodified pass when among them. Fixed preserves option order; Rotated applies a nonzero cyclic rotation on prompt 2. Both models use 32 programs and the same 4,413 items with equal task weights. $\Delta$ = Real $-$ \RDf{}. Bold marks the larger displayed Real/\RDf{} estimate.
\end{minipage}
\end{table}

\paragraph{Input-blind increments produce substantial headroom.}
The real programs change the answer on \val{qwen.churn.min}--\val{qwen.churn.max}\% of item--prompt pairs on Qwen3-4B-Base and \val{llama.churn.min}--\val{llama.churn.max}\% on Llama-3.1-8B.
On evaluation items, \RDf{} matches all 32 programs within the churn tolerance on both models, with median ratios \val{qwen.val.median} and \val{llama.val.median}.
For the original draw with shared option order, the families' same-prompt headroom differs by at most 0.1 point in displayed point estimate, and both have large cross-prompt headroom (\cref{tab:large}).
The placebo reaches the half-real threshold on both models: $C=\val{qwen.C}$~\val{qwen.C.ci} and $\val{llama.C}$~\val{llama.C.ci} points, with clustered lower bounds \val{qwen.C.clo} and \val{llama.C.clo}.
Across the full menus, at least one real program is correct on \val{qwen.v12.rescue.real}\% and \val{llama.v12.rescue.real}\% of the item--prompt pairs that the unmodified model gets wrong, and at least one placebo on \val{qwen.v12.rescue.rd}\% and \val{llama.v12.rescue.rd}\%.
The half-real threshold is also exceeded in both fresh draws per model (\cref{sec:redraw}).
Thus substantial oracle headroom, including cross-prompt gains, is possible with an input-blind injected increment.

\paragraph{The ordering depends on the menu.}
\RDf{} has larger headroom on the full and skip-only menus in every draw (\cref{fig:menus} shows draw 1).
On Llama's repeat-only menu, real programs lead in all three draws; on Qwen's repeat menu, the original interval includes zero but draws 2--3 favour \RDf{}.
Llama's real advantage on the low-churn menu is detected only in draw 1; its later intervals and all three Qwen intervals include zero (\cref{tab:submenus}).
The half-real contrast $C$ has a positive pointwise 95\% lower bound in every model, submenu and draw, including the menus restricted to real calibration churn at most 30\%.
These are exploratory comparisons, rather than simultaneous confidence claims across submenus.
Submenus overlap and do not decompose full-menu headroom.
Per-task results, menu-size curves and the margin ranges contributing most to the gains are reported in \cref{app:large}.

\begin{figure}[t]
\centering
\includegraphics[width=\textwidth]{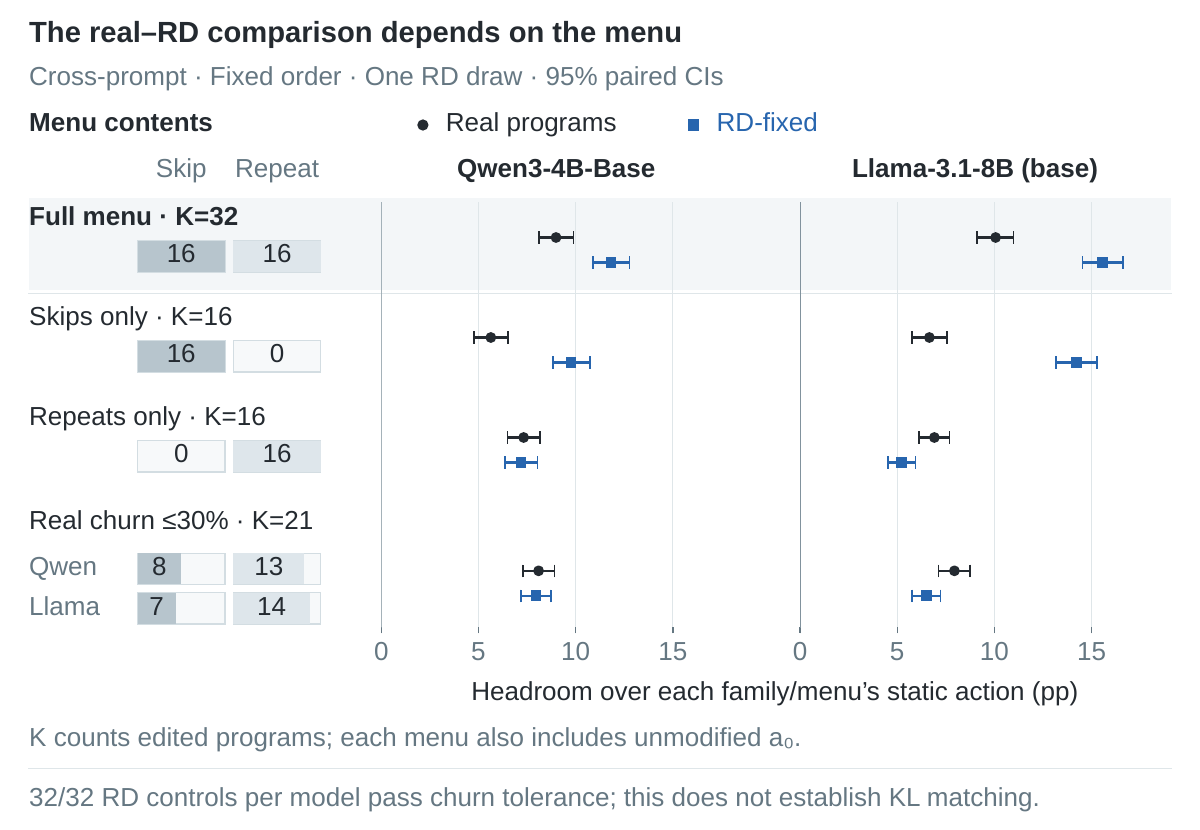}
\caption{\textbf{Cross-prompt headroom depends on the action menu.}
Left: counts of skip and repeat programs, each on a common 0--16 scale. Every menu also includes $a_0$.
Low-churn membership uses the real programs' answer-change rates of at most 30\% on independent calibration items; the paired RD menu uses the same program IDs within each model.
Right: headroom over each family/menu/task's development-selected static action, with fixed option order, tie averaging and equal task weights.
Points and 95\% task-stratified paired item-bootstrap intervals use \val{L.items} held-out items per model and the original RD draw.
Menus overlap; the models share the program-grid rule but use different physical layer segments.
Calibration matches churn only: program-level median KL(RD)/median KL(Real) ranges from 0.19 to 4.84 for Qwen and from 0.22 to 7.71 for Llama, measured against the unmodified full-vocabulary output at the answer position.}
\label{fig:menus}
\end{figure}

\paragraph{What the comparison establishes.}
Passing the churn rule does not make \RDf{} a distributionally matched replacement for the real programs.
The families differ in KL scale and cross-prompt variability (\cref{app:theory}), and the original RD draw requires scales up to \val{qwen.cal.c.max} and \val{llama.cal.c.max} times the real update's median norm; Qwen's draws 2 and 3 need at most \val{redraw@qwen.d23.scale.max} times, yet still exceed real-program headroom (\cref{app:redraw}).
In the original draw, a joint covariance test rejects a testable implication of the matched-placebo assumptions on both models after Holm correction (\cref{app:matching-check}).
Its shared directions can also favour cross-prompt reliability.
The result therefore establishes what this specified control can produce, without estimating the non-specific share of real-program headroom.

\subsection{Shared answer structure and option rotation}
\label{sec:mechanism}

\paragraph{A sufficient source of cross-prompt gains.}
An action that consistently increases one answer letter's log-probability can help the same item on both prompts when the correct option keeps that letter.
An oracle can select such an action using the answer even though the increment itself does not read the question.
To measure this possible source of headroom, the analysis uses a menu of input-independent letter offsets, fitted on development items separately for each action, task, option count and position, and evaluated with its own static reference.
Their point estimates range from \val{L.oos.pct.min} to \val{L.oos.pct.max}\% of their source families' headroom; the four paired intervals are reported individually in \cref{tab:offset-ratios} (original RD draw).
Fixed letter offsets are therefore sufficient to produce headroom of this scale.
They do not give a decomposition of the observed gains: subtracting mean letter offsets still leaves substantial headroom (\cref{app:large}).

\paragraph{Changing the shared option order.}
Rotating prompt 2's options reduces cross-prompt headroom by \val{qwen.v11.Grepro_drop_real} and \val{llama.v11.Grepro_drop_real} points for the real programs, and by \val{qwen.v11.Grepro_drop_rd} and \val{llama.v11.Grepro_drop_rd} for the placebo.
Same-prompt headroom and prompt-2 unmodified accuracy change by at most about one point in point estimate (\cref{fig:permuted,app:large}).
The placebo still passes the churn rule.
Under tie averaging, both families' cross-prompt headroom is negative relative to their development-selected static actions, with all four 95\% intervals below zero.
For the real programs it is \val{qwen.v11.Grepro_real_perm}~\val{qwen.v11.Grepro_real_perm.ci} points on Qwen and \val{llama.v11.Grepro_real_perm}~\val{llama.v11.Grepro_real_perm.ci} on Llama.
This level has no calibrated null interpretation: every correct option moves, so a letter preference that helped on prompt 1 can hurt on prompt 2.
Keeping $a_0$ on accuracy ties instead gives real-program headroom of \val{qwen.v11.Grepro_a0_real_perm}~\val{qwen.v11.Grepro_a0_real_perm.ci} and \val{llama.v11.Grepro_a0_real_perm}~\val{llama.v11.Grepro_a0_real_perm.ci} points, respectively (\cref{tab:large}), while retaining the large drop under rotation.

\begin{figure}[t]
\centering
\includegraphics[width=\textwidth]{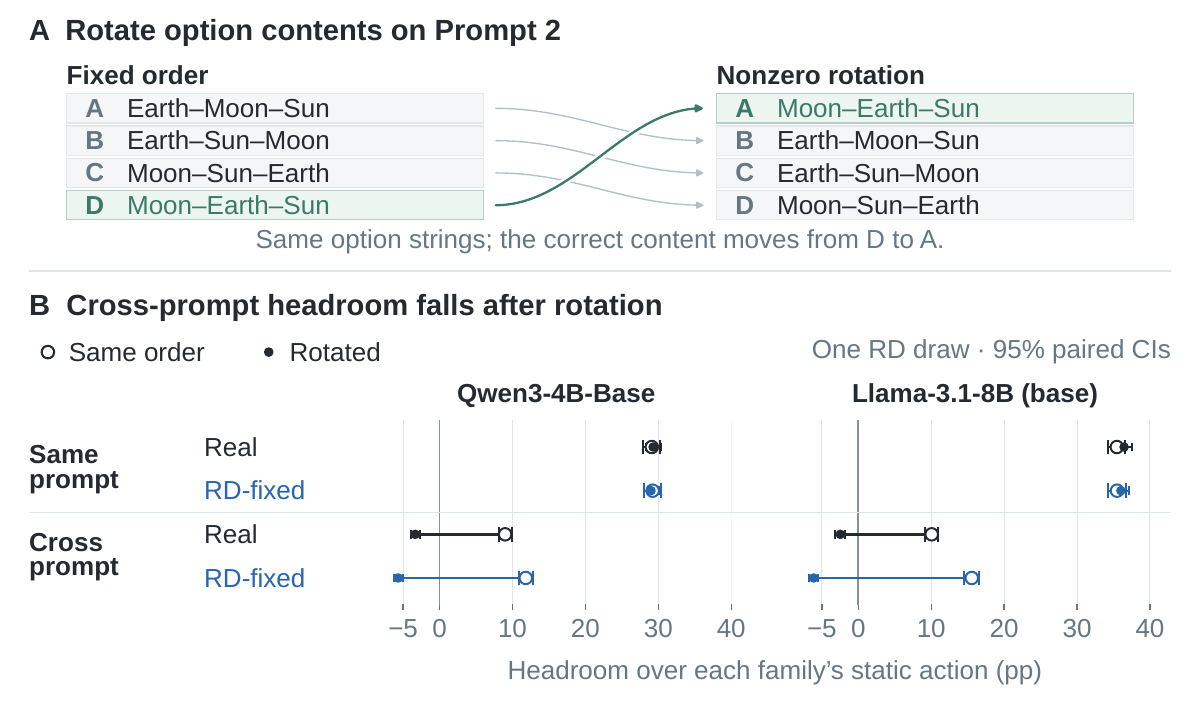}
\caption{\textbf{Option-order changes affect cross-prompt headroom.}
(A)~The lunar-eclipse example from \cref{fig:overview} illustrates a nonzero cyclic rotation on prompt 2: correct content moves from D to A. This is a schematic intervention, not a new measured outcome for that item.
(B)~Same-prompt and cross-prompt headroom over each family's development-selected static action.
Open circles show shared order; filled circles show rotated order, with horizontal segments spanning the measured differences on a common numerical scale.
Results use the full 32-program menu plus $a_0$, tie averaging, \val{L.items} held-out items per model and equal task weights.
Intervals are 95\% task-stratified paired item-bootstrap intervals (original RD draw).
Zero denotes equality with the static reference, not a validated null for computation-specific effects.
\Cref{fig:design-reference} gives the mixture and reference sensitivity analyses, including Llama's item-decoupled mixture contrast whose interval contains zero.}
\label{fig:permuted}
\end{figure}

\paragraph{The residual comparison remains conditional on the design.}
With the primary tie-average rule and static references, real programs' headroom is less negative than \RDf{} after rotation:
$\Delta=\val{qwen.pm.D}$~\val{qwen.pm.D.ci} and $\val{llama.pm.D}$~\val{llama.pm.D.ci} points (clustered intervals: \val{qwen.v14.Dperm.cci} and \val{llama.v14.Dperm.cci}).
This is a difference between two headrooms, each measured against its own family's static action.
It does not isolate computation-specific effects: the placebo has stronger average letter preferences, so rotation can penalise it more.
After subtracting development-fitted mean letter offsets and choosing static actions on the adjusted development scores, rotated $\Delta$ is \val{qwen.adj.main.Delta_adj_perm}~\val{qwen.adj.main.Delta_adj_perm.ci} and \val{llama.adj.main.Delta_adj_perm}~\val{llama.adj.main.Delta_adj_perm.ci} points (draw 1; \cref{tab:adjusted}).
This adjustment changes the comparison; it does not isolate a computation-specific remainder.

A sensitivity analysis reweights the stored outputs to retain the original order with probability $1/n$ for an item with $n$ options.
Under the same static references, the real-minus-RD headroom difference decreases to \val{qwen.cpu.mix.D} and \val{llama.cpu.mix.D} points.
The real family's headroom is \val{qwen.cpu.mix.Greal}~\val{qwen.cpu.mix.Greal.ci} points on Qwen and \val{llama.cpu.mix.Greal}~\val{llama.cpu.mix.Greal.ci} on Llama: the first is negative and the second has an interval containing zero.
Using an item-decoupled reference instead asks a different question, comparing the selected action with selections borrowed from other items.
Under that reference, the interval for Llama's mixture contrast includes zero (\cref{fig:design-reference,app:large}).
The mixture approximately balances correct letters across items, but does not establish a universal zero point for the controls.
Together these analyses show substantial sensitivity to shared answer structure, while leaving the source of the remaining family difference unresolved.
Full results for alternative references and ties, and the semi-synthetic decision-calibration and power analyses, are in \cref{app:large}.

\subsection{Variation across draws and control constructions}

\paragraph{Random-direction redraws.}
\label{sec:redraw}
Each model has three independent draws of the random bases, with scales recalibrated and placebo static actions reselected for each draw; all pass the churn matching rules in both designs.
Across the draws, $C$ ranges from \val{redraw@qwen.C1.range} points for Qwen and from \val{redraw@llama.C1.range} for Llama, and rotated $\Delta$ from \val{redraw@qwen.D_perm.range} and from \val{redraw@llama.D_perm.range} points; every draw has positive 95\% lower bounds for both contrasts under item and cluster resampling (\cref{app:redraw}).
The draws share items and real-program outputs. Treating directions as random, a rough interval keeps $C>0$ but includes zero for Qwen's fixed-order and rotated $\Delta$ (\cref{app:redraw}).
Redrawing checks sensitivity to random directions, not distributional matching or the source of $\Delta$.

\paragraph{Additional controls in a smaller menu.}
\label{sec:ladder}
A supplementary comparison evaluates \RDi{}, \SH{} and \Rot{} for three Qwen3-4B-Base programs: skip layers 16--19, repeat 16--19 and repeat 28--31, on MMLU-Pro, BBH and AQuA-RAT \citep{ling2017aqua} (300 calibration and 1,800 evaluation items; \cref{app:ladder}).
All three families pass the median-KL rule and reach \val{S.rcross.min}--\val{S.rcross.max}\% of the real programs' cross-prompt headroom in point estimate.
Real-program point estimates are higher, but no real-minus-placebo contrast survives Holm correction under either static-action rule (\cref{tab:ladder}).
The intervals permit both no real advantage and advantages of several points; they do not establish equivalence.
Because this study changes the control construction, matched statistic, menu, tasks and model coverage together, it cannot isolate which change accounts for the different ordering.

\section{Transfer of selected problem--program associations}
\label{sec:transfer}

Beyond letter-scored multiple choice, a practical question remains: does a program selected for a problem retain an advantage over programs selected for other problems when the wording changes?
A supplementary experiment tests this with generated math answers, using programs from a public reconstruction of PoLar's search \citep{westerhoff2026programs}.
The comparison measures transfer of the selected problem--program association and has no placebo.

The transfer study uses the released search logs for the post-trained Qwen3-8B release \citep{yang2025qwen3} on GSM8K-sourced DART-Math problems \citep{tong2024dartmath,cobbe2021gsm8k} that the unmodified model answers wrongly and the search rescues.
The study uses each problem's first rescuing program and retains problems whose rescue was re-confirmed (\val{T.reconf.pct}\% were) and whose rewordings passed mechanical checks and two automatic LLM reviewers (\cref{app:transfer}).
Problems are sorted by the number of executed layers and cut into consecutive groups of four.
On each kept rewording, the problem’s own program is compared with the three other programs in its group (\emph{donors}); $\Tres$ is the own program's accuracy minus the donors' mean, averaged within and then across problems, with intervals that resample whole groups.
After an exploratory pilot, a confirmation on 200 fresh problems was pre-specified, with success defined by a 95\% interval for $\Tres$ whose lower end exceeded 10 percentage points.

\paragraph{Results and interpretation.}
Own programs retained a correct answer on \val{T.v15.own_accuracy}\%~\val{T.v15.own_accuracy.ci} of the kept rewordings (post hoc group-bootstrap interval); each had answered its original problem correctly by selection.
Donors answered \val{T.acc.donor}\% and the unmodified model \val{T.acc.identity}\% correctly (\cref{tab:transfer} in \cref{app:transfer}).
$\Tres=\val{T.T}$ points (95\% interval \val{T.ci95}) meets the pre-specified criterion.
Post hoc paired comparisons give own minus unmodified of \val{T.v15.own_minus_unmodified}~\val{T.v15.own_minus_unmodified.ci} points and unmodified minus donors of \val{T.v15.unmodified_minus_donors}~\val{T.v15.unmodified_minus_donors.ci}.
Thus most of the point-estimated own--donor advantage is a gain over the unmodified model; the donor deficit's interval includes zero.
The advantage remains in groups whose programs execute the same number of layers (\val{T.strict.T}~\val{T.strict.ci}) and, in the pre-specified paired secondary comparison on the \val{T.style.n} problems with both styles retained, is smaller on restructured than on sentence-by-sentence rewordings (by \val{T.style.diff} points \val{T.style.ci}).

The selected association therefore transfers within this pool of rescued problems, despite the large loss from the original, selected 100\% correctness.
The style gap is compatible with sensitivity to surface wording; it does not isolate which wording features matter.
Its source remains open: a fixed perturbation selected for a problem might also transfer, and the rewordings preserve the answer, allowing a stable preference for the answer string to persist. This study tested neither explanation.
Each program was selected by an answer-informed search on the original wording, so it does not show that a selector could find useful programs for new problems.

\section{Related work}
\label{sec:related}

\paragraph{Per-input adaptive computation.}
Per-input chains of skipped and repeated layers are searched by CoLa \citep{li2025cola}, predicted by PoLar \citep{li2026polar} and used to supervise routers by Dr.LLM \citep{heakl2026drllm}, whereas MACRO learns one route distribution per task \citep{batorski2026macro}; related work learns per-token recursion depth \citep{bae2025mor}, loops blocks without training \citep{chen2026tflt} or reuses decoder blocks to diversify sampled candidates \citep{singh2026archsampling}, and early-exit work uses oracle exits as bounds or supervision \citep{elbayad2020depth,schuster2022calm}.
Frozen models tolerate deleting and reordering layers \citep{gromov2025unreasonable,lad2025robustness,sun2025painters}, although replacing middle layers with copies of one layer is harmful \citep{sun2025painters}.

\paragraph{Interpreting oracle gains.}
In routing among models, \citet{chen2026routing} separates the hindsight oracle from the single-commit ceiling under stochastic decoding, and \citet{shihab2026opportunity} give selection-valid intervals for the oracle gap and separate opportunity from the gain realizable from a pre-answer signal; identical copies of one generated LLM harness already yield oracle headroom \citep{xu2026harness}.
\citet{li2026reachability} study question-level reachability over layer routes with generated answers and report that random routes match or exceed structured route search under a matched budget; their random routes are themselves layer programs.
\citet{westerhoff2026programs} reconstruct PoLar's search across five models and report strong option-order sensitivity of rescuing programs on MMLU-Pro and that their reproduced router's top predictions concentrate on the identity program.
The transfer test in \cref{sec:transfer} uses their released programs.
This study compares layer programs with perturbation controls injected at the same site.
These controls replace the intended edit. For \RD, the injected update reads nothing from the current input, making it possible to examine what headroom can arise without that computation.

\paragraph{Sanity checks, selection and controls.}
Placebo comparisons follow work that examines what widely used measurements show \citep{adebayo2018sanity,schaeffer2023mirage,agarwal2021precipice}.
The best of several noisy options is optimistic \citep{smith2006optimizer,andrews2024inference}, as are a classifier pool's per-instance oracle \citep{souza2017oracle} and a virtual best solver from single runs \citep{cameron2016bias}; the discount on re-measurement is regression toward the mean \citep{kelley1927}.
The placebos studied here are active placebos, related to negative controls \citep{lipsitch2010negative,sofer2016negative}; $\Delta$ and $\rho$ compare real actions with named controls, as selectivity compares a probe with control tasks \citep{hewitt2019control}.
In interpretability, noise and resampled corruptions lead to different conclusions \citep{meng2022rome,chan2022causal,zhang2024patching}, and another input's layers serve as a steering control \citep{nelwan2026steering}; random perturbations also raise best-of-several performance \citep{gan2026thickets}, sharply so for sandbagging models \citep{tice2025noise}.
Multiple-choice scoring carries option-letter and option-order biases \citep{zheng2024mcq,pezeshkpour2024order}, and few-shot predictions input-independent label biases \citep{zhao2021calibrate}; in the present setting they are one route to headroom.

\section{Discussion}
\label{sec:discussion}

\paragraph{Interpreting oracle gains.}
Shared answer structure can support substantial headroom, as the input-blind controls, fixed letter offsets and option rotation show.
Residual differences can reflect computation-specific effects, unaccounted non-specific sensitivity, or both.
The same-prompt oracle bounds a router's accuracy over its menu, without identifying the source of that opportunity.
Across the two models and two reported tie rules, the largest full-menu rotated headroom is \val{qwen.v11.Grepro_a0_real_perm} points, with interval \val{qwen.v11.Grepro_a0_real_perm.ci}, relative to the development-selected static actions.
This motivates checking option-order sensitivity before using such oracle choices as router labels; cross-prompt headroom is not an upper bound on a learned router's gain.

\paragraph{Recommendations.}
Report the action budget, static reference and tie rule, together with the control construction, calibration statistic and held-out matching checks.
Use independent control draws to assess direction sensitivity, and re-measure selections after changing relevant shared structure, such as option order.
A held-out letter-offset menu tests one sufficient source of headroom; \cref{app:ladder} gives an analogous output-rescaling check for margin scores.
A persistent real-program advantage can support specificity relative to the tested controls only under the assumptions that justify that comparison.

\paragraph{Limitations.}
The placebo evidence in this study covers direct-answer multiple choice scored by letter log-probabilities, base models of 4B and 8B parameters, and menus of at most 32 single-segment programs, including programs that change most answers.
This study does not evaluate these controls on CoLa or PoLar's published systems, generative oracles, or oracles defined by compute savings or agreement with the full model.
The transfer test covers one model, short direct answers to grade-school problems and automatically reviewed rewordings within a pool selected for successful rescue.
The design with rotated option order does not separate answer letters from option positions and has no natural null.
The controls are calibrated on one statistic, which does not establish distributional matching.
The large-menu covariance test rejects an implication of the stronger matched-placebo assumptions (\cref{app:matching-check}).
In the original large-menu draw, \RDf{} has lower prompt-to-prompt noise (median variance ratios \val{qwen.diag.noise} and \val{llama.diag.noise}), which can favour cross-prompt reliability; its construction also gives the skip and repeat of a segment opposite directions.
Letter-offset, covariance and tie-type diagnostics use the original draw; three draws give limited evidence about sensitivity to random directions.
Intervals condition on fitted controls and prompts, and both models share items.
Public benchmarks may appear in pretraining data; this study has not measured how such exposure affects these comparisons.

\section{Conclusion}
With shared option order, input-blind controls exceeded real-program headroom in every full-menu draw, while Llama's repeat-only menu favoured real programs in a post hoc analysis.
Rotation sharply reduced headroom and left a positive full-menu real-minus-control difference.
Neither ordering identifies a computation-specific benefit: in letter-scored multiple choice, large gains with shared option order are insufficient evidence by themselves.
These controlled-menu findings do not test CoLa or PoLar's published systems.
In the supplementary generated-answer test, selected programs answered \val{T.v15.own_accuracy}\% of rewordings correctly (donors: \val{T.acc.donor}\%), without a placebo comparison.

\FloatBarrier
\bibliographystyle{tmlr}
\bibliography{refs}

\newpage
\appendix
\section{Theory}
\label{app:theory}

\begin{table}[htbp]
\centering
\begin{minipage}{396pt}\approvedtablefont
\caption{\textbf{Notation used in the comparisons.}}
\label{tab:notation}
\begingroup\fontsize{9}{11}\selectfont
\setlength{\tabcolsep}{4pt}\renewcommand{\arraystretch}{1.15}
\begin{tabular*}{\linewidth}{@{\extracolsep{\fill}}p{109pt}p{275pt}@{}}
\toprule
Symbol & Meaning \\
\midrule
$a_0,F,K$ & Unmodified pass, action family, and number of edited candidates (excluding $a_0$). \\
$U_i^v(a),\Acc_i^v(a)$ & Correct-letter margin and 0/1 correctness for item $i$, prompt $v$ and action $a$. \\
$A_i^s,a_S$ & Best-accuracy tie set on selection prompt $s$; static action selected without evaluation outcomes. \\
$\Graw,\Grepro$ & Headroom over the static action when selection and scoring use the same prompt, or different prompts. \\
$C$ & Fixed-order half-headroom contrast: $\Grepro(\RDm)-\tfrac12\Grepro(\Real)$. \\
$\Delta,\rho$ & Real minus control headroom; control divided by real headroom. \\
$\mu_k,g_{ik},h_{ik},e_{ik}^v$ & Mean effect; centered stable non-specific and specific effects; prompt-varying noise. \\
$c_k,r_k(t),\hat u_{k,t}$ & Calibrated scale, norm profile, and fixed random unit direction (one per distance from the sequence end) of the main control. \\
$\Tres$ & Accuracy advantage of a rescued problem's selected program over other problems' donor programs after rewording. \\
\bottomrule
\end{tabular*}\endgroup
\end{minipage}
\end{table}

\subsection{Argument for Observation~\ref{prop:nonid}}
Let the real family's observable data be the margins $U_i^v(a_k)$ for prompts $v=1,\dots,V$ (any $V\ge1$, including repeated draws or seeds).
Under \cref{eq:model}, $U_i^v(a_k)=U_i^v(a_0)+\mu_k+s_{ik}+e_{ik}^v$ with $s=g+h$, so the joint distribution of the observables is a function of the joint distribution of $(U(a_0),\mu,s,e)$.
Any two decompositions $s=g+h=g'+h'$ with the same joint distribution of $s$---for example $(g',h')=(s,0)$, no specific component, and $(0,s)$, an entirely specific one---induce the same distribution of every observable, so every functional of the observables takes the same value under both.
\hfill$\square$

Under the stated assumptions on stable effects and prompt-specific noise, resampling \citep{chen2026routing}, multiple seeds and cross-prompt evaluation can characterize the stable component, but do not separate $g$ from $h$.
Besides a comparison family, other designs can add the missing information, for example multiple proxies \citep{miao2018proxy}, items on which the specific effect is known to be zero, or empirical calibration against many controls \citep{schuemie2014calibration}.

\subsection{What a matched placebo identifies}
A placebo family has effects $x^{N,v}_{ik}=\mu^N_k+g^N_{ik}+e^{N,v}_{ik}$.

\begin{proposition}
\label{prop:matched}
Suppose that, for action $k$, $g^N_{\cdot k}$ has the same distribution as $g_{\cdot k}$ and $e^N$ as $e$, that $g$ is independent of $h$, and that the noise is independent of the stable effects and across prompts.
Then the cross-prompt covariance of per-item effects is $\operatorname{Var}(g_k)$ for the placebo and $\operatorname{Var}(g_k)+\operatorname{Var}(h_k)$ for the real action, so their ratio is the non-specific share of stable variance $\phi_k$, and $\operatorname{cov}_{N,k}\le\operatorname{cov}_{\Real,k}$.
If, jointly with the unmodified margins $U(a_0)$, the law of $(\mu^N,g^N,e^N)$ across actions equals that of $(\mu,g,e)$, then $\Grepro(N)$ has the distribution of the real family's cross-prompt headroom in the counterfactual without $h$.
For this headroom statement, the joint-law condition also covers any development outcomes used to select the static action; the same static-action selection and tie rules are applied in the placebo and no-$h$ counterfactual.
\end{proposition}

\emph{Proof.}
The noise contributes nothing to the cross-prompt covariance.
For the real action, $\operatorname{cov}(x_k^1,x_k^2)=\operatorname{Var}(g_k)+\operatorname{Var}(h_k)+2\operatorname{Cov}(g_k,h_k)$, and the last term vanishes because $g$ is independent of $h$; the placebo's covariance is $\operatorname{Var}(g^N_k)=\operatorname{Var}(g_k)$.
Under the joint condition the placebo family's observables have the distribution the real family's would have with $h\equiv0$.
The baseline margins must be part of the condition, because accuracy is $\mathbf 1[U(a_0)+x>0]$ and its headroom concentrates on items near the threshold (\cref{fig:margin}).
The counterfactual reapplies the stated static-action selection rule after removing $h$.
\hfill$\square$

In causal-inference terms, the placebos studied here are active placebos rather than negative controls in the sense of \citet{lipsitch2010negative}.
The matching condition is a variance analogue of the equi-confounding assumption behind negative outcome controls \citep{sofer2016negative}.

\subsection{An exchangeable Gaussian example}
The example concerns a margin-scale oracle, without the accuracy threshold; it is a benchmark under its assumptions and does not decide what an observed accuracy ratio means.
Take exchangeable actions with $\mu=0$ and independent normal $g$, $h$, $e$ with variances $\sigma_g^2,\sigma_h^2,\sigma_e^2$, and let $m_K=\mathbb E\max(0,Z_1,\dots,Z_K)$ for standard normal $Z$.
Selecting on one prompt and scoring on the other gives $\Grepro(\Real)=m_K(\sigma_g^2+\sigma_h^2)/\tau_R$ and $\Grepro(N)=m_K\sigma_g^2/\tau_N$, with $\tau_R^2=\sigma_g^2+\sigma_h^2+\sigma_e^2$ and $\tau_N^2=\sigma_g^2+\sigma_e^2$, so
\begin{equation}
\rho=\phi\,\tau_R/\tau_N\in[\phi,\sqrt\phi],\qquad \phi=\sigma_g^2/(\sigma_g^2+\sigma_h^2).
\end{equation}
Absolute headroom grows with the menu like an order statistic ($m_K=0.40,0.89,1.42,2.07$ for $K=1,3,8,32$), but $\rho$ does not depend on $K$.
Two consequences matter for reading data: $\rho\ge\phi$, so even a matched placebo's $\rho$ overstates the non-specific share; and $\rho\le\sqrt\phi\le1$, so in this model a $\rho$ above one is incompatible with a matched placebo.
A Monte Carlo check with $10^6$ items per cell ($\phi\in\{1,0.75,0.5,0.25\}$, reliability $\in\{1,0.85,0.5\}$, $K\in\{1,3,32\}$) reproduced the closed forms to within Monte Carlo error, and every simulated $\rho$ lay in $[\phi,\sqrt\phi]$.

\subsection{The matching condition in the observed data}
\label{app:matching-check}
The joint condition fails in the three-program study: every placebo's mean effect differs from the real action's, and four of the nine noise-variance ratios exclude one (\cref{tab:jointcheck}).
Under the condition, \cref{prop:matched} implies that no placebo's cross-prompt covariance of per-item effects exceeds the real action's; the placebo-to-real covariance ratios for the skip and the two repeats are \val{S.RDS.covratio} (\RDi), \val{S.SHS.covratio} (\SH) and \val{S.ROT.covratio} (\Rot), so five of the six repeat ratios exceed one in point estimate, by up to 0.12. A subsequent paired-bootstrap max-$T$ test of the joint hypothesis that every placebo-minus-real covariance difference is non-positive does not reject it in this study ($p=\val{S.cpu.cov.p}$); this does not establish matching. A ratio above one contradicts the matching condition together with the independence of $g$ and $h$, not the matching condition alone.
In the large-menu study, \RDf{} has a lower noise variance than the real programs (median ratio over programs \val{qwen.diag.noise}~\val{qwen.diag.noise.ci} on Qwen3-4B-Base and \val{llama.diag.noise}~\val{llama.diag.noise.ci} on Llama-3.1-8B), and a larger cross-prompt covariance for \val{qwen.diag.covpos} and \val{llama.diag.covpos} of the 32 programs.
For each model, the analysis tests $H_0:D_k\leq0$ for every program, where $D_k$ is placebo minus real cross-prompt covariance.
The statistic is $\max_k \widehat D_k/\widehat{\mathrm{se}}_k$, using \val{cpu.cov.B} paired item-bootstrap replicates stratified by task; the reference distribution is the maximum of the centered bootstrap differences divided by the same standard errors.
The covariance pools items across tasks, including task-mean differences.
The observed statistics are \val{qwen.cpu.cov.T} and \val{llama.cpu.cov.T}; both $p$-values are approximately \val{cpu.cov.mc.resolution}, the minimum attainable under the plus-one Monte Carlo rule.
Holm correction across the \val{cpu.cov.holm.n} joint tests (the two large-menu models and the three-program study) rejects both large-menu nulls and not the three-program null.
This check excludes the single BBH item with only one answer option, for which a correct-versus-incorrect margin is undefined. Per-program covariances, ratios and intervals \ifanonymous are included in the supplementary\else will be released with the other\fi{} numerical results.
Placebo-to-real ratios are therefore comparisons with named placebos, not shares.

\begin{table}[htbp]
\centering
\begin{minipage}{396pt}\approvedtablefont
\caption{\textbf{Mean effects and noise variances.}}
\label{tab:jointcheck}
\begingroup\fontsize{8.5}{10}\selectfont
\setlength{\tabcolsep}{3pt}\renewcommand{\arraystretch}{1.04}
\begin{tabular*}{\linewidth}{@{\extracolsep{\fill}}llcc@{}}
\toprule
\textbf{Action} & \textbf{Placebo} & \shortstack{\textbf{Mean difference (nats)}\\placebo $-$ real; ref. 0} & \shortstack{\textbf{Noise variance ratio}\\placebo / real; ref. 1} \\
\midrule
\textbf{skip 16--19} & \RDi & \ci{+0.19}{+0.14}{+0.24} & \ci{0.80}{0.72}{0.91} \\
 & \SH & \ci{+0.16}{+0.11}{+0.21} & \ci{1.08}{0.91}{1.26} \\
 & \Rot & \ci{+0.11}{+0.06}{+0.16} & \ci{0.96}{0.85}{1.08} \\
\addlinespace[3pt]
\textbf{repeat 16--19} & \RDi & \ci{-0.11}{-0.16}{-0.06} & \ci{0.92}{0.83}{1.02} \\
 & \SH & \ci{-0.05}{-0.09}{-0.004} & \ci{0.71}{0.65}{0.78} \\
 & \Rot & \ci{-0.08}{-0.13}{-0.04} & \ci{0.85}{0.76}{0.94} \\
\addlinespace[3pt]
\textbf{repeat 28--31} & \RDi & \ci{-0.07}{-0.10}{-0.04} & \ci{1.11}{1.02}{1.22} \\
 & \SH & \ci{-0.05}{-0.08}{-0.03} & \ci{0.95}{0.88}{1.02} \\
 & \Rot & \ci{-0.09}{-0.12}{-0.06} & \ci{1.12}{0.98}{1.29} \\
\bottomrule
\end{tabular*}\endgroup

\par\vspace{6pt}
{\fontsize{8.5}{10}\selectfont\raggedright
Pooled three-program study; estimate [95\% item-bootstrap CI]. Mean differences use margin effects in nats; noise variance is $\operatorname{Var}(x^1-x^2)/2$. The real mean effects are $-0.90$, $-0.06$, and $-0.06$ nats for skip 16--19, repeat 16--19, and repeat 28--31, respectively. Reference values 0 and 1 denote equality with the real action.\par}
\end{minipage}
\end{table}

\FloatBarrier
\section{Placebo construction, calibration and implementation}
\label{app:placebos}

\subsection{Large-menu study}

\paragraph{Programs.}
Each model's menu is a $2\times4\times4$ grid: skip or repeat once; segment length 1--4; one segment centred in each of four equal-width depth bands.
A program named \texttt{skip2@19} skips the two layers starting at layer 19 (0-indexed); \texttt{repeat3@10} runs layers 10--12 twice.

\paragraph{\RDf.}
All randomness comes from 64-bit seeds taken from SHA-256 hashes of fixed strings.
Within each draw, for every model and every layer $\ell$ modified by some program, a $J\times D$ standard Gaussian matrix $B_\ell$ ($J=2{,}048$) is generated once; row $j$ is used at the position $j$ tokens before the end of the sequence, for every item and prompt.
A program $p$ with sign $s_p=-1$ (skip) or $+1$ (repeat) and modified layers $S_p$ has direction $u_{p,t}=s_p\sum_{\ell\in S_p}b_{\ell,j(t)}$, normalized per position.
The injected update is $c_p\,r_p(t)\,\hat u_{p,t}$ (the scale $c_k$ of \cref{sec:method}, indexed here by program $p$), where $r_p(t)$ is the median norm of the real update at the same distance from the end on 1,000 calibration items (the first position, an attention sink, has its own median; distances observed on fewer than 50 item-prompt pairs use the median over all positions).

\paragraph{Calibration.}
Scales are $c_p=2^{q_p/4}$.
For each program, the placebo's top-1 answer-change rate on the 1,000 calibration items was measured on a coarse grid $q\in\{-16,-12,\dots,24\}$ and refined twice around the best point; $c_p$ minimizes the absolute difference between the placebo's and the real program's number of changed answers (ties to the smaller scale; no interpolation).
The tolerance is $\max(0.25\cdot\text{churn}_{\Real},\tau)$, with $\tau=\max(1\text{ point},3f)$.
The numerical floor $f$ is the larger of two answer-change rates: between the Hugging Face forward pass and the explicit layer-loop implementation, and between two independent executions.
On both models $f=\val{qwen.cal.f}$, so $\tau=1$ point.
On the calibration items all 32 programs were within tolerance on both models (largest absolute differences \val{qwen.cal.maxdiff} and \val{llama.cal.maxdiff} points; medians \val{qwen.cal.meddiff} and \val{llama.cal.meddiff}), no scale lay at a grid end ($c_p$ from \val{qwen.cal.c.min} to \val{qwen.cal.c.max} on Qwen3-4B-Base and from \val{llama.cal.c.min} to \val{llama.cal.c.max} on Llama-3.1-8B), and every program's answer-change rate exceeded $5\tau$, so all 32 formed the active set of the matching rule.
KL divergence was not matched; the placebo-to-real ratio of median KL on the evaluation items ranges from \val{qwen.kl.min} to \val{qwen.kl.max} (Qwen3-4B-Base) and from \val{llama.kl.min} to \val{llama.kl.max} (Llama-3.1-8B).

\paragraph{Matching rule on the evaluation items.}
A family passes if at least 30 of 32 programs are within tolerance, all 32 have a churn ratio in $[0.5,2]$ or an absolute difference of at most $2\tau$, and the median churn ratio over the active set lies in $[0.9,1.1]$.
It uses answer changes only, not labels.
Both models' RD families passed: 32/32 programs within tolerance, ratios \val{qwen.val.ratio.min}--\val{qwen.val.ratio.max} and \val{llama.val.ratio.min}--\val{llama.val.ratio.max}, medians \val{qwen.val.median} and \val{llama.val.median}.
\subsection{Three-program study}
Calibration used 300 items (600 item-prompt pairs) disjoint from the evaluation items, and read no labels.
For each placebo and action, the median KL divergence $D_{\mathrm{KL}}(p_0\,\|\,p)$ over the vocabulary at the answer position was measured at eight scales $c\in\{0.5,0.75,1,1.5,2,3,4,6\}$, and $c_k$ was interpolated log-linearly to the real action's median KL.
All nine scales were interior (\cref{tab:laddercal}).
The \RDi{} norm profile $r_k(t)$ was fitted on the calibration items' real updates as a function of the distance from the end.
\SH{} maps receiver position $t\ge1$ to donor rows aligned from the end, so both prompts of an item receive identical donor rows on their shared suffix; no position receives the receiver's own update, and donor few-shot examples are disjoint from the receiver's.
\Rot{} uses one $2560\times2560$ Haar-random orthogonal matrix; it preserves per-token norms and the inner products between the three actions' updates.

\begin{table}[htbp]
\centering
\begin{minipage}{396pt}\approvedtablefont
\caption{\textbf{Calibration and matching check.}}
\label{tab:laddercal}
\begingroup\fontsize{8.5}{10}\selectfont
\setlength{\tabcolsep}{3pt}\renewcommand{\arraystretch}{1.04}
\begin{tabular*}{\linewidth}{@{\extracolsep{\fill}}llcccc@{}}
\toprule
\multirow{2}{*}{\textbf{Action}} & \multirow{2}{*}{\textbf{Placebo}} & \multirow{2}{*}{\textbf{Target KL}} & \multirow{2}{*}{\shortstack{\textbf{Scale}\\$c_k$}} & \multicolumn{2}{c}{\textbf{KL ratio (placebo / real)}} \\
\cmidrule(lr){5-6}
 &  &  &  & \textbf{Calibration} & \textbf{Evaluation} \\
\midrule
\textbf{skip 16--19} & \RDi & \num{0.496} & \num{1.607} & \num{1.011} & \num{1.007} \\
 & \SH & \num{0.496} & \num{1.014} & \num{1.026} & \num{1.015} \\
 & \Rot & \num{0.496} & \num{1.551} & \num{1.005} & \num{0.981} \\
\addlinespace[3pt]
\textbf{repeat 16--19} & \RDi & \num{0.0753} & \num{1.137} & \num{0.997} & \num{1.096} \\
 & \SH & \num{0.0753} & \num{0.980} & \num{0.997} & \num{1.011} \\
 & \Rot & \num{0.0753} & \num{0.933} & \num{0.995} & \num{1.024} \\
\addlinespace[3pt]
\textbf{repeat 28--31} & \RDi & \num{0.0347} & \num{1.213} & \num{0.985} & \num{1.064} \\
 & \SH & \num{0.0347} & \num{0.964} & \num{1.007} & \num{0.901} \\
 & \Rot & \num{0.0347} & \num{1.013} & \num{1.010} & \num{1.043} \\
\bottomrule
\end{tabular*}\endgroup

\par\vspace{6pt}
{\fontsize{8.5}{10}\selectfont\raggedright
Three-program study. Target KL is the real action's median KL on the calibration items; $c_k$ is the fitted placebo scale. Ratios compare placebo-to-real median KL on calibration and evaluation items, with equality at 1.\par}
\end{minipage}
\end{table}

A family passed if all three ratios lay in $[0.8,1.25]$ on the evaluation items; all three were.
Matching the median KL left the placebos' margin damage close to the real actions' (0.320--0.359 against 0.339 nats) but not their accuracy damage (7.0--7.8 against 9.4 points), and their effects are slightly less dispersed (standard-deviation ratios 0.88--0.91).

\subsection{Implementation}
\label{app:verification}
All actions share one unmodified pass cached at segment boundaries, and every program reproduces a plain forward pass over its executed layer sequence exactly.
The head is applied in fp32 to the bf16 final hidden state, because bf16 logits put letter margins on a grid of $1/16$ nat or coarser and create exact ties.
Every candidate is computed at batch size 1, because batch composition alone changes the top-1 letter on 2--4\% of item-prompt pairs.

\FloatBarrier
\section{Large-menu study: data and further results}
\label{app:large}

\paragraph{Analysis specification.}
\label{app:analysis-timing}
Before the development outcomes were observed, the large-menu study specified the half-headroom hypothesis and a stronger $\pm20\%$ equivalence hypothesis, along with the menu, tie averaging, development-selected static actions and churn calibration.
The half-headroom threshold was a practical criterion for a substantial control effect, not a theoretical boundary for computation-specific need.
The prospective power rule stopped the initial study before evaluation because the available pool could not adequately power equivalence.
After the development estimates were inspected, a separate confirmation of the existing half-headroom hypothesis was specified on the untouched evaluation pool; its threshold and the action, calibration and scoring rules were unchanged.
The unmodified-pass reference, original letter-shift diagnostic and menu-size curves were also specified as descriptive checks.
Later additions include option rotation, keeping $a_0$ on accuracy ties, item-decoupled and mixture comparisons, development-fitted letter-offset menus, and the paired intervals and tie-type decomposition reported here.
The two fresh RD draws were specified after the original findings and before their new runs; the transfer confirmation followed its own exploratory pilot and was specified before evaluating its fresh problems.
\Cref{tab:analysis-status} summarizes these distinctions. \ifanonymous The supplement contains anonymized\else The release will include\fi{} English versions of the original internally versioned plans and amendments, indexed in \texttt{analysis\_records/INDEX.txt}; they were not registered publicly, and their recorded dates are not independently certified timestamps.

\begin{table}[htbp]
\centering
\begin{minipage}{396pt}\approvedtablefont
\caption{\textbf{Role and timing of the analyses.}}
\label{tab:analysis-status}
\begingroup\fontsize{9}{11}\selectfont
\setlength{\tabcolsep}{4pt}\renewcommand{\arraystretch}{1.12}
\begin{tabular*}{\linewidth}{@{\extracolsep{\fill}}>{\raggedright\arraybackslash}p{128pt}>{\raggedright\arraybackslash}p{255pt}@{}}
\toprule
Analysis & Role and specification relative to outcomes \\
\midrule
Initial half-headroom and equivalence hypotheses & Specified before development outcomes; the original joint power requirement stopped evaluation. No equivalence claim is made. \\
Half-headroom confirmation & Confirmation specified after development estimates and before untouched evaluation; the threshold, menu and scoring rules were retained. \\
Original descriptive checks & Unmodified-pass reference, original letter-shift diagnostic, menu-size curves and cluster sensitivity were planned. \\
Rotation and other diagnostic extensions & Exploratory relative to the original study: rotated options, alternative ties, fitted letter offsets, decoupled and mixture references, covariance tests, and later paired decompositions. \\
Two fresh RD draws & Robustness contrasts specified after original findings and before the new draws; evaluation items were reused. \\
Three-draw means and expanded submenus & Post hoc summaries of stored outputs, with conditional pointwise intervals and no new confirmation rule. \\
Fresh-item transfer confirmation & Primary own--donor threshold and secondary paired style contrast specified after the pilot and before fresh-item runs; same-length and accuracy summaries were planned supplements. \\
Transfer paired decomposition & Own minus unmodified, unmodified minus donors, and the retention-rate interval were added after confirmation. \\
\bottomrule
\end{tabular*}\endgroup
\end{minipage}
\end{table}

\paragraph{Items and prompts.}
Items come from MMLU-Pro (test split), the BBH tasks with letter answers and at most ten options, and ARC-Challenge, excluding every item and few-shot example used elsewhere in this work, items with more than ten options, and items whose prompt exceeds 2,048 tokens under either tokenizer.
Items were split by stratified seeded sampling into 1,000 calibration items (label-free) and 1,000 development items in total (333 or 334 per task) and, within each task, an evaluation order; ARC-Challenge calibration and development items come from its training split and its evaluation items from the validation and test splits.
The evaluation set is the first \val{L.nd} items of each task's order, the size of the smallest pool, and the same items are used for both models.
Each item has four distinct few-shot examples drawn from the same MMLU-Pro category or BBH task (ARC: 100 reserved training items), two per prompt; each answer letter is scored as a single token with a leading space. The last examples of the two prompts share their answer letter for \val{L.v12.fewshot.last}\% of items, and \val{L.v12.fewshot.contains}\% of prompts contain an example whose answer letter is the item's correct letter.
One evaluation item (BBH \texttt{snarks}) has a single option because of a parsing error; every action receives accuracy 1 under the parsed representation.
It remains in the accuracy denominator with zero per-item headroom, and is excluded from analyses requiring a correct-versus-incorrect margin.

\paragraph{Cross-prompt scoring.}
The prompts differ only in their two few-shot examples; the question, the options and their order are identical in the fixed-order design.
For each item, the oracle averages over the actions that answer the selection prompt correctly (over all actions if none does), scores them on the other prompt, and averages the two directions.
Cross-prompt headroom subtracts the accuracy of the development-selected static action under the same scoring rule.
\Cref{fig:example} shows the answer-letter probabilities of the example in the overview.

\begin{figure}[!ht]
\centering
\begin{minipage}{396pt}
\includegraphics[width=\linewidth]{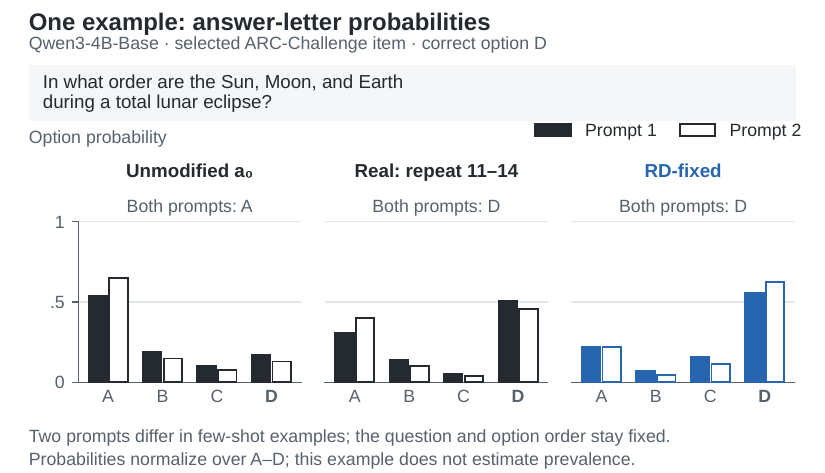}
\caption[One example: answer-letter probabilities]{Answer-letter probabilities for the lunar-eclipse example (Qwen3-4B-Base, ARC-Challenge; original RD draw). Options are A: Earth--Moon--Sun, B: Earth--Sun--Moon, C: Moon--Sun--Earth, D: Moon--Earth--Sun (correct). Filled and open bars use different few-shot examples with identical question/options. The unmodified model predicts A on both; repeating layers 11--14 and its paired RD-fixed control predict D on both. Probabilities normalize over the four letters on a shared 0--1 scale. This selected example illustrates output changes, not population frequencies.}
\label{fig:example}
\end{minipage}
\end{figure}

\paragraph{Static actions.}
The static action of each model, task and family is the action with the highest mean accuracy over both prompts on the development items (ties to $a_0$, then to the lowest program index).
\paragraph{Per task and alternative baselines.}
\Cref{tab:largetasks} gives the unmodified model's accuracy and the cross-prompt headroom per task, with the chosen static actions and with $a_0$ as the static action (point estimates).
Resampling the 14 MMLU-Pro categories and 15 BBH tasks as clusters (ARC's \val{L.nd} items singly) gives $\Grepro(\Real)$ \val{qwen.Greal.cci} and \val{llama.Greal.cci}, $\Grepro(\RDm)$ \val{qwen.Grd.cci} and \val{llama.Grd.cci}, and $C$ \val{qwen.C.cci} and \val{llama.C.cci} on Qwen3-4B-Base and Llama-3.1-8B.
The difference $\Delta=\Grepro(\Real)-\Grepro(\RDm)$ is \val{qwen.D}~\val{qwen.D.ci} and \val{llama.D}~\val{llama.D.ci} (item bootstrap).

\begin{table}[htbp]
\centering
\begin{minipage}{396pt}\approvedtablefont
\caption{\textbf{Large-menu study: results by task.}}
\label{tab:largetasks}
\begingroup\fontsize{8.5}{10}\selectfont
\setlength{\tabcolsep}{4pt}\renewcommand{\arraystretch}{1.04}
\begin{tabular*}{\linewidth}{@{\extracolsep{\fill}}lccccc@{}}
\toprule
 &  & \multicolumn{2}{c}{$\Grepro$ (pp): chosen static} & \multicolumn{2}{c}{$\Grepro$ (pp): static $a_0$} \\
\cmidrule(lr){3-4}\cmidrule(lr){5-6}
Task & $\mathrm{Acc}(a_0)$ (\%) & Real & \RDf & Real & \RDf \\
\midrule
\addlinespace[2pt]
\multicolumn{6}{@{}l}{\textbf{Qwen3-4B-Base}} \\[1pt]
MMLU-Pro & \num{45.2} & \num{12.7} & \numb{16.3} & \num{12.5} & \numb{16.3} \\
BBH & \num{51.5} & \num{11.0} & \numb{15.4} & \num{10.7} & \numb{15.4} \\
ARC-Challenge & \num{89.9} & \num{3.3} & \numb{3.7} & \num{2.8} & \numb{3.1} \\
\addlinespace[2pt]
\multicolumn{6}{@{}l}{\textbf{Llama-3.1-8B}} \\[1pt]
MMLU-Pro & \num{34.2} & \num{12.6} & \numb{17.9} & \num{11.7} & \numb{18.2} \\
BBH & \num{40.7} & \num{11.2} & \numb{19.2} & \num{9.9} & \numb{19.5} \\
ARC-Challenge & \num{80.3} & \num{6.5} & \numb{9.6} & \num{5.5} & \numb{8.8} \\
\bottomrule
\end{tabular*}\endgroup

\par\vspace{6pt}
{\fontsize{8.5}{10}\selectfont\raggedright
Fixed option order; original RD draw. $\mathrm{Acc}(a_0)$ is unmodified-model accuracy averaged over prompts (\%). $\Grepro$ is cross-prompt oracle accuracy minus static-reference accuracy (pp). Chosen static maximizes development accuracy for each model, task and family; the alternative reference is $a_0$. All entries are point estimates; bold marks the larger displayed Real/\RDf{} estimate in each pair.\par}
\end{minipage}
\end{table}

\paragraph{Single programs.}
With a single alternative per oracle, \RDf{} has higher cross-prompt headroom than the real program for \val{qwen.prog.skip.rdwins} and \val{llama.prog.skip.rdwins} of the 16 skips on Qwen3-4B-Base and Llama-3.1-8B and lower for \val{qwen.prog.rep.realwins} and \val{llama.prog.rep.realwins} of the 16 repeats (point estimates); its largest leads occur for early skips of three or four layers that change more than half of the answers.
The oracle over all 32 programs is a maximum, not a sum, so these values do not decompose the menu-level difference.

\paragraph{Sub-menus.}
\Cref{tab:submenus,tab:submenus-rotated} restrict the menu to the 16 repeats, the 16 skips, or the programs whose real answer-change rate on the calibration items is at most 30\%, with static actions re-chosen within each sub-menu on the development items and kept for the rotated design.
These values are descriptive; the oracle over the full menu is a maximum, so they do not decompose its headroom.
Under rotation, 16 of the 18 pointwise intervals for $\Delta$ exclude zero on the positive side. The exceptions are Qwen\textquotesingle s low-churn menu in draws 2 and 3: \storedci{qwen.v16.sub.churn_le_30.d2.D_perm} and \storedci{qwen.v16.sub.churn_le_30.d3.D_perm} points, respectively.

\begin{table}[htbp]
\centering
\begin{minipage}{396pt}\approvedtablefont
\caption{\textbf{Submenu headroom across draws: shared option order.}}
\label{tab:submenus}
\begingroup\fontsize{8.5}{10}\selectfont
\setlength{\tabcolsep}{2pt}\renewcommand{\arraystretch}{1.10}
\begin{tabular*}{\linewidth}{@{\extracolsep{\fill}}llcccc@{}}
\toprule
Menu ($K$) & Draw & Real & \RDf & $\Delta$ & $C$ \\
\midrule
\multicolumn{6}{@{}l}{\textbf{Qwen3-4B-Base}} \\[1pt]
Repeats (16) & 1 & \storedcib{qwen.v16.sub.repeats.d1.G_real_fixed} & \storedci{qwen.v16.sub.repeats.d1.G_rd_fixed} & \storedci{qwen.v16.sub.repeats.d1.D_fixed} & \storedci{qwen.v16.sub.repeats.d1.C} \\
 & 2 & \storedci{qwen.v16.sub.repeats.d2.G_real_fixed} & \storedcib{qwen.v16.sub.repeats.d2.G_rd_fixed} & \storedci{qwen.v16.sub.repeats.d2.D_fixed} & \storedci{qwen.v16.sub.repeats.d2.C} \\
 & 3 & \storedci{qwen.v16.sub.repeats.d3.G_real_fixed} & \storedcib{qwen.v16.sub.repeats.d3.G_rd_fixed} & \storedci{qwen.v16.sub.repeats.d3.D_fixed} & \storedci{qwen.v16.sub.repeats.d3.C} \\
\addlinespace[3pt]
Skips (16) & 1 & \storedci{qwen.v16.sub.skips.d1.G_real_fixed} & \storedcib{qwen.v16.sub.skips.d1.G_rd_fixed} & \storedci{qwen.v16.sub.skips.d1.D_fixed} & \storedci{qwen.v16.sub.skips.d1.C} \\
 & 2 & \storedci{qwen.v16.sub.skips.d2.G_real_fixed} & \storedcib{qwen.v16.sub.skips.d2.G_rd_fixed} & \storedci{qwen.v16.sub.skips.d2.D_fixed} & \storedci{qwen.v16.sub.skips.d2.C} \\
 & 3 & \storedci{qwen.v16.sub.skips.d3.G_real_fixed} & \storedcib{qwen.v16.sub.skips.d3.G_rd_fixed} & \storedci{qwen.v16.sub.skips.d3.D_fixed} & \storedci{qwen.v16.sub.skips.d3.C} \\
\addlinespace[3pt]
Churn $\leq30\%$ (21) & 1 & \storedcib{qwen.v16.sub.churn_le_30.d1.G_real_fixed} & \storedci{qwen.v16.sub.churn_le_30.d1.G_rd_fixed} & \storedci{qwen.v16.sub.churn_le_30.d1.D_fixed} & \storedci{qwen.v16.sub.churn_le_30.d1.C} \\
 & 2 & \storedci{qwen.v16.sub.churn_le_30.d2.G_real_fixed} & \storedci{qwen.v16.sub.churn_le_30.d2.G_rd_fixed} & \storedci{qwen.v16.sub.churn_le_30.d2.D_fixed} & \storedci{qwen.v16.sub.churn_le_30.d2.C} \\
 & 3 & \storedci{qwen.v16.sub.churn_le_30.d3.G_real_fixed} & \storedcib{qwen.v16.sub.churn_le_30.d3.G_rd_fixed} & \storedci{qwen.v16.sub.churn_le_30.d3.D_fixed} & \storedci{qwen.v16.sub.churn_le_30.d3.C} \\
\addlinespace[3pt]
\multicolumn{6}{@{}l}{\textbf{Llama-3.1-8B}} \\[1pt]
Repeats (16) & 1 & \storedcib{llama.v16.sub.repeats.d1.G_real_fixed} & \storedci{llama.v16.sub.repeats.d1.G_rd_fixed} & \storedci{llama.v16.sub.repeats.d1.D_fixed} & \storedci{llama.v16.sub.repeats.d1.C} \\
 & 2 & \storedcib{llama.v16.sub.repeats.d2.G_real_fixed} & \storedci{llama.v16.sub.repeats.d2.G_rd_fixed} & \storedci{llama.v16.sub.repeats.d2.D_fixed} & \storedci{llama.v16.sub.repeats.d2.C} \\
 & 3 & \storedcib{llama.v16.sub.repeats.d3.G_real_fixed} & \storedci{llama.v16.sub.repeats.d3.G_rd_fixed} & \storedci{llama.v16.sub.repeats.d3.D_fixed} & \storedci{llama.v16.sub.repeats.d3.C} \\
\addlinespace[3pt]
Skips (16) & 1 & \storedci{llama.v16.sub.skips.d1.G_real_fixed} & \storedcib{llama.v16.sub.skips.d1.G_rd_fixed} & \storedci{llama.v16.sub.skips.d1.D_fixed} & \storedci{llama.v16.sub.skips.d1.C} \\
 & 2 & \storedci{llama.v16.sub.skips.d2.G_real_fixed} & \storedcib{llama.v16.sub.skips.d2.G_rd_fixed} & \storedci{llama.v16.sub.skips.d2.D_fixed} & \storedci{llama.v16.sub.skips.d2.C} \\
 & 3 & \storedci{llama.v16.sub.skips.d3.G_real_fixed} & \storedcib{llama.v16.sub.skips.d3.G_rd_fixed} & \storedci{llama.v16.sub.skips.d3.D_fixed} & \storedci{llama.v16.sub.skips.d3.C} \\
\addlinespace[3pt]
Churn $\leq30\%$ (21) & 1 & \storedcib{llama.v16.sub.churn_le_30.d1.G_real_fixed} & \storedci{llama.v16.sub.churn_le_30.d1.G_rd_fixed} & \storedci{llama.v16.sub.churn_le_30.d1.D_fixed} & \storedci{llama.v16.sub.churn_le_30.d1.C} \\
 & 2 & \storedci{llama.v16.sub.churn_le_30.d2.G_real_fixed} & \storedcib{llama.v16.sub.churn_le_30.d2.G_rd_fixed} & \storedci{llama.v16.sub.churn_le_30.d2.D_fixed} & \storedci{llama.v16.sub.churn_le_30.d2.C} \\
 & 3 & \storedci{llama.v16.sub.churn_le_30.d3.G_real_fixed} & \storedcib{llama.v16.sub.churn_le_30.d3.G_rd_fixed} & \storedci{llama.v16.sub.churn_le_30.d3.D_fixed} & \storedci{llama.v16.sub.churn_le_30.d3.C} \\
\addlinespace[3pt]
\bottomrule
\end{tabular*}\endgroup
\par\vspace{6pt}
{\fontsize{8.5}{10}\selectfont\raggedright
Headroom and contrasts are in percentage points; $\Delta=\Grepro(\Real)-\Grepro(\RDm)$ and $C=\Grepro(\RDm)-\tfrac12\Grepro(\Real)$. Each menu includes $a_0$ in addition to $K$ edited programs. The low-churn menu uses the same original real-calibration program IDs in all draws. Static actions are selected on development items within each menu and family for each draw, and retained under rotation. Brackets are pointwise 95\% paired item-bootstrap intervals, stratified by task and conditional on each fitted control. These post hoc comparisons are not simultaneous tests. Bold marks the larger displayed family point estimate; rounded ties are unmarked. Overlapping menus do not decompose full-menu headroom.\par}
\end{minipage}
\end{table}

\begin{table}[htbp]
\centering
\begin{minipage}{396pt}\approvedtablefont
\caption{\textbf{Submenu headroom across draws: rotated option order.}}
\label{tab:submenus-rotated}
\begingroup\fontsize{8.5}{10}\selectfont
\setlength{\tabcolsep}{2pt}\renewcommand{\arraystretch}{1.10}
\begin{tabular*}{\linewidth}{@{\extracolsep{\fill}}llccc@{}}
\toprule
Menu ($K$) & Draw & Real & \RDf & $\Delta$ \\
\midrule
\multicolumn{5}{@{}l}{\textbf{Qwen3-4B-Base}} \\[1pt]
Repeats (16) & 1 & \storedcib{qwen.v16.sub.repeats.d1.G_real_perm} & \storedci{qwen.v16.sub.repeats.d1.G_rd_perm} & \storedci{qwen.v16.sub.repeats.d1.D_perm} \\
 & 2 & \storedcib{qwen.v16.sub.repeats.d2.G_real_perm} & \storedci{qwen.v16.sub.repeats.d2.G_rd_perm} & \storedci{qwen.v16.sub.repeats.d2.D_perm} \\
 & 3 & \storedcib{qwen.v16.sub.repeats.d3.G_real_perm} & \storedci{qwen.v16.sub.repeats.d3.G_rd_perm} & \storedci{qwen.v16.sub.repeats.d3.D_perm} \\
\addlinespace[3pt]
Skips (16) & 1 & \storedcib{qwen.v16.sub.skips.d1.G_real_perm} & \storedci{qwen.v16.sub.skips.d1.G_rd_perm} & \storedci{qwen.v16.sub.skips.d1.D_perm} \\
 & 2 & \storedcib{qwen.v16.sub.skips.d2.G_real_perm} & \storedci{qwen.v16.sub.skips.d2.G_rd_perm} & \storedci{qwen.v16.sub.skips.d2.D_perm} \\
 & 3 & \storedcib{qwen.v16.sub.skips.d3.G_real_perm} & \storedci{qwen.v16.sub.skips.d3.G_rd_perm} & \storedci{qwen.v16.sub.skips.d3.D_perm} \\
\addlinespace[3pt]
Churn $\leq30\%$ (21) & 1 & \storedcib{qwen.v16.sub.churn_le_30.d1.G_real_perm} & \storedci{qwen.v16.sub.churn_le_30.d1.G_rd_perm} & \storedci{qwen.v16.sub.churn_le_30.d1.D_perm} \\
 & 2 & \storedcib{qwen.v16.sub.churn_le_30.d2.G_real_perm} & \storedci{qwen.v16.sub.churn_le_30.d2.G_rd_perm} & \storedci{qwen.v16.sub.churn_le_30.d2.D_perm} \\
 & 3 & \storedci{qwen.v16.sub.churn_le_30.d3.G_real_perm} & \storedci{qwen.v16.sub.churn_le_30.d3.G_rd_perm} & \storedci{qwen.v16.sub.churn_le_30.d3.D_perm} \\
\addlinespace[3pt]
\multicolumn{5}{@{}l}{\textbf{Llama-3.1-8B}} \\[1pt]
Repeats (16) & 1 & \storedcib{llama.v16.sub.repeats.d1.G_real_perm} & \storedci{llama.v16.sub.repeats.d1.G_rd_perm} & \storedci{llama.v16.sub.repeats.d1.D_perm} \\
 & 2 & \storedcib{llama.v16.sub.repeats.d2.G_real_perm} & \storedci{llama.v16.sub.repeats.d2.G_rd_perm} & \storedci{llama.v16.sub.repeats.d2.D_perm} \\
 & 3 & \storedcib{llama.v16.sub.repeats.d3.G_real_perm} & \storedci{llama.v16.sub.repeats.d3.G_rd_perm} & \storedci{llama.v16.sub.repeats.d3.D_perm} \\
\addlinespace[3pt]
Skips (16) & 1 & \storedcib{llama.v16.sub.skips.d1.G_real_perm} & \storedci{llama.v16.sub.skips.d1.G_rd_perm} & \storedci{llama.v16.sub.skips.d1.D_perm} \\
 & 2 & \storedcib{llama.v16.sub.skips.d2.G_real_perm} & \storedci{llama.v16.sub.skips.d2.G_rd_perm} & \storedci{llama.v16.sub.skips.d2.D_perm} \\
 & 3 & \storedcib{llama.v16.sub.skips.d3.G_real_perm} & \storedci{llama.v16.sub.skips.d3.G_rd_perm} & \storedci{llama.v16.sub.skips.d3.D_perm} \\
\addlinespace[3pt]
Churn $\leq30\%$ (21) & 1 & \storedcib{llama.v16.sub.churn_le_30.d1.G_real_perm} & \storedci{llama.v16.sub.churn_le_30.d1.G_rd_perm} & \storedci{llama.v16.sub.churn_le_30.d1.D_perm} \\
 & 2 & \storedcib{llama.v16.sub.churn_le_30.d2.G_real_perm} & \storedci{llama.v16.sub.churn_le_30.d2.G_rd_perm} & \storedci{llama.v16.sub.churn_le_30.d2.D_perm} \\
 & 3 & \storedcib{llama.v16.sub.churn_le_30.d3.G_real_perm} & \storedci{llama.v16.sub.churn_le_30.d3.G_rd_perm} & \storedci{llama.v16.sub.churn_le_30.d3.D_perm} \\
\addlinespace[3pt]
\bottomrule
\end{tabular*}\endgroup
\par\vspace{6pt}
{\fontsize{8.5}{10}\selectfont\raggedright
Headroom and contrasts are in percentage points; $\Delta=\Grepro(\Real)-\Grepro(\RDm)$. Each menu includes $a_0$ in addition to $K$ edited programs. The low-churn menu uses the same original real-calibration program IDs in all draws. Static actions are selected on development items within each menu and family for each draw, and retained under rotation. Brackets are pointwise 95\% paired item-bootstrap intervals, stratified by task and conditional on each fitted control. These post hoc comparisons are not simultaneous tests. Bold marks the larger displayed family point estimate; rounded ties are unmarked. Overlapping menus do not decompose full-menu headroom.\par}
\end{minipage}
\end{table}

\paragraph{Menu size and letter-offset controls.}
\Cref{fig:menu-size} shows how fixed-order headroom changes with the number of edited actions and compares the full menus with their development-fitted letter-offset menus.

\begin{table}[htbp]
\centering
\begin{minipage}{396pt}\approvedtablefont
\caption{\textbf{Letter-offset menus relative to their source families.}}
\label{tab:offset-ratios}
\begingroup\fontsize{8.5}{10}\selectfont
\setlength{\tabcolsep}{4pt}\renewcommand{\arraystretch}{1.12}
\newcommand{\offsetci}[1]{\shortstack{\val{#1}\\[-1pt]{\fontsize{7.5}{9}\selectfont\val{#1.ci}}}}
\begin{tabular*}{\linewidth}{@{\extracolsep{\fill}}llccc@{}}
\toprule
Model & Source family & \shortstack{Source headroom\\(pp)} & \shortstack{Offset headroom\\(pp)} & \shortstack{Offset / source\\(\%)} \\
\midrule
Qwen3-4B-Base & Real & \offsetci{qwen.v15.offset.real.G_family} & \offsetci{qwen.v15.offset.real.G_offset_menu} & \offsetci{qwen.v15.offset.real.ratio} \\
 & \RDf & \offsetci{qwen.v15.offset.rd.G_family} & \offsetci{qwen.v15.offset.rd.G_offset_menu} & \offsetci{qwen.v15.offset.rd.ratio} \\
\addlinespace[3pt]
Llama-3.1-8B & Real & \offsetci{llama.v15.offset.real.G_family} & \offsetci{llama.v15.offset.real.G_offset_menu} & \offsetci{llama.v15.offset.real.ratio} \\
 & \RDf & \offsetci{llama.v15.offset.rd.G_family} & \offsetci{llama.v15.offset.rd.G_offset_menu} & \offsetci{llama.v15.offset.rd.ratio} \\
\bottomrule
\end{tabular*}\endgroup
\par\vspace{6pt}
{\fontsize{8.5}{10}\selectfont\raggedright
Original RD draw, shared option order, full menus and tie averaging. Each menu uses its own development-selected static action. Brackets are individual 95\% paired, task-stratified item-bootstrap intervals; each resample recomputes the ratio from the paired numerator and denominator, with fitted offsets and static actions held fixed. These post hoc intervals do not establish equivalence or a decomposition of headroom.\par}
\end{minipage}
\end{table}

\begin{figure}[!ht]
\centering
\begin{minipage}{396pt}
\includegraphics[width=\linewidth]{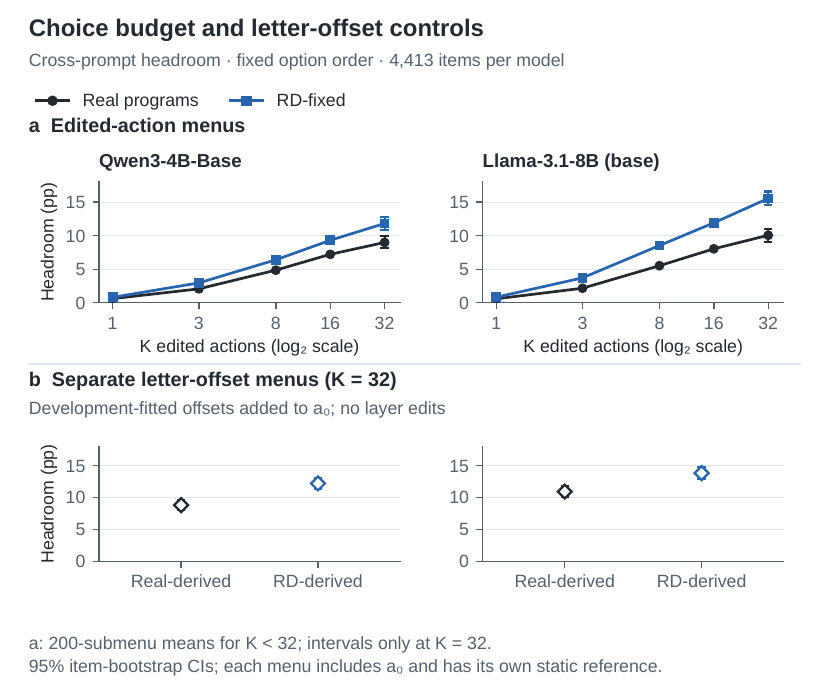}
\caption[Choice budget and letter-offset controls]{Cross-prompt headroom over development-selected static actions, with fixed option order (4,413 held-out items per model; original RD draw). (a) K counts edited actions, with a$_0$ included separately; the horizontal axis is logarithmic (base 2). For K \textless{} 32, points average 200 paired random submenus with static actions reselected within each submenu. Only the full-menu points have 95\% item-bootstrap intervals; these do not measure variation across submenus. (b) Separate 32-action menus add each family's development-fitted letter offsets to a$_0$'s scores, with their own development-selected static actions; bars are 95\% item-bootstrap intervals. Both panels share the scale, weight tasks equally and stratify item resampling by task.}
\label{fig:menu-size}
\end{minipage}
\end{figure}

\paragraph{Where the headroom arises.}
Gains are concentrated among items with a slightly negative mean unmodified margin across the two prompts and are small where the mean margin is positive (\cref{fig:margin}); a negative mean margin does not imply a wrong answer on both prompts.

\begin{figure}[!ht]
\centering
\begin{minipage}{396pt}
\includegraphics[width=\linewidth]{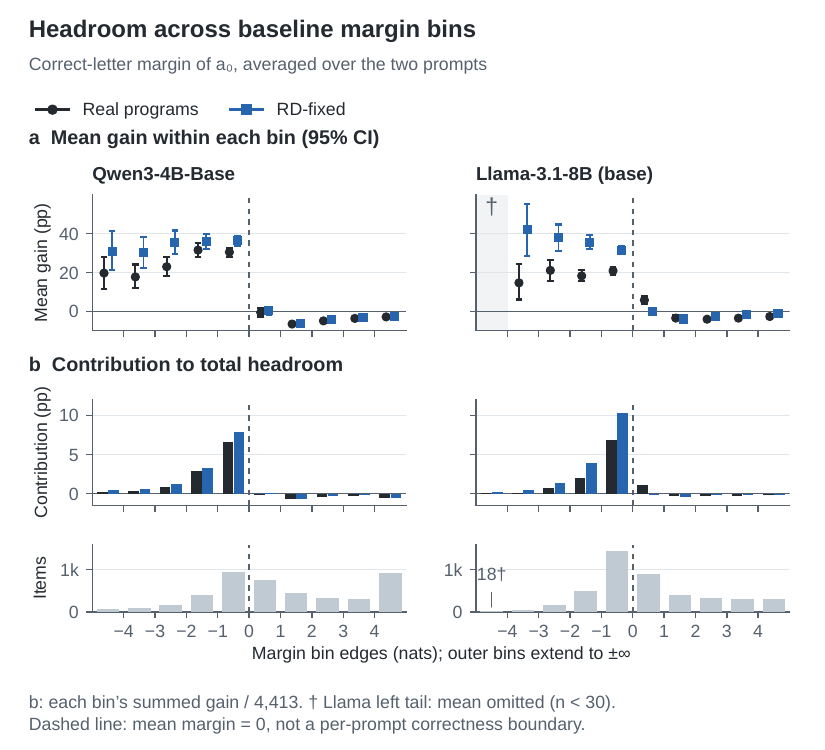}
\caption[Headroom across baseline margin bins]{Fixed-order cross-prompt gains by the unmodified correct-letter margin averaged over prompts (original RD draw). Margin is the correct letter's log probability minus the largest incorrect letter's log probability. (a) Mean oracle-minus-static accuracy gain within each bin, with pointwise 95\% within-bin item-bootstrap intervals (2,000 replicates). The shaded Llama left tail (\textdagger{}) has 18 items; its mean and interval are omitted because n \textless{} 30. (b) Each bin's summed gain divided by all 4,413 items; bars sum to the family's total headroom, including the left tail. The bottom row shows item counts on a shared scale; outer bins extend to infinity. Mean margin sign does not imply identical correctness on both prompts.}
\label{fig:margin}
\end{minipage}
\end{figure}

\FloatBarrier
\paragraph{Letter-offset menus.}
For each action $a$ of a family, task and number of options $n$, let $\bar o_a(\ell)$ be the mean over items with $n$ options and both prompts of $\log p_a(\ell)-\log p_{a_0}(\ell)$ for letter position $\ell\le n$.
The letter-offset action adds $\bar o_a$ to $a_0$'s letter log-probabilities on every item.
In the version reported in the main text, $\bar o_a$ is estimated on the development items and the menu's static action is chosen there by the rule used for the families; its $\Grepro$ on the evaluation items is \val{qwen.oos.real} and \val{qwen.oos.rd} points for the offsets of the real programs and of \RDf{} on Qwen3-4B-Base and \val{llama.oos.real} and \val{llama.oos.rd} on Llama-3.1-8B (\val{qwen.oos.real.fam}, \val{qwen.oos.rd.fam}, \val{llama.oos.real.fam} and \val{llama.oos.rd.fam} with the families' static actions), and intervals condition on the fitted offsets. The percentages in the main text are ratios of unrounded values.
Evaluation items whose number of options does not occur among the development items of their task (\val{qwen.oos.missing} items) keep $a_0$'s log-probabilities.
Estimated on the evaluation items instead, the menu reaches \val{qwen.ls.real} and \val{qwen.ls.rd} points on Qwen3-4B-Base and \val{llama.ls.real} and \val{llama.ls.rd} on Llama-3.1-8B. With the options of prompt 2 rotated, the development-fitted menu reaches \val{qwen.v10.S_real_perm} and \val{qwen.v10.S_rd_perm} points on Qwen3-4B-Base and \val{llama.v10.S_real_perm} and \val{llama.v10.S_rd_perm} on Llama-3.1-8B (\cref{tab:permuted}).

\paragraph{Letter-offset-adjusted oracle.}
The adjusted oracle subtracts from every action's letter log-probabilities its offsets $\bar o_a$ estimated on the development items, leaves $a_0$ unchanged, and is otherwise computed as for the families, in both designs; the \val{qwen.adj.missing} evaluation items whose number of options does not occur among the development items of their task are not adjusted.
Its static action is chosen on the development items from the adjusted scores, the rule used for the letter-offset menu; the families' own static actions and $a_0$ are sensitivity analyses (\cref{tab:adjusted}).
Subtracting the mean offsets does not remove item-specific letter preferences, and the adjusted oracle does not separate specific from non-specific effects.

\begin{table}[htbp]
\centering
\begin{minipage}{396pt}\approvedtablefont
\caption{\textbf{Large-menu study: letter-offset adjustment.}}
\label{tab:adjusted}
\begingroup\fontsize{8.5}{10}\selectfont
\setlength{\tabcolsep}{1pt}\renewcommand{\arraystretch}{1.02}
\begin{tabular*}{\linewidth}{@{\extracolsep{\fill}}lccc@{\hspace{4pt}}ccc@{}}
\toprule
 & \multicolumn{3}{c}{Fixed: $\Grepro$ (pp)} & \multicolumn{3}{c}{Rotated: $\Grepro$ (pp)} \\
\cmidrule(lr){2-4}\cmidrule(lr){5-7}
Static & Real & \RDf & $\Delta$ & Real & \RDf & $\Delta$ \\
\midrule
\addlinespace[2pt]
\multicolumn{7}{@{}l}{\textbf{Qwen3-4B-Base}} \\[1pt]
Dev. & \cib{7.3}{6.4}{8.2} & \ci{7.0}{6.2}{7.8} & \ci{0.3}{-0.6}{1.2} & \cib{-3.1}{-3.7}{-2.4} & \ci{-4.0}{-4.6}{-3.4} & \ci{0.9}{0.2}{1.6} \\
Family & \ci{7.6}{6.7}{8.4} & \cib{7.9}{7.0}{8.7} & \ci{-0.3}{-1.2}{0.6} & \cib{-3.0}{-3.6}{-2.3} & \ci{-3.7}{-4.3}{-3.1} & \ci{0.7}{0.03}{1.4} \\
$a_0$ & \ci{7.0}{6.2}{7.8} & \ci{7.0}{6.2}{7.8} & \ci{0.0}{-0.6}{0.6} & \cib{-3.5}{-4.2}{-2.9} & \ci{-4.3}{-4.9}{-3.7} & \ci{0.8}{0.3}{1.2} \\
\addlinespace[2pt]
\multicolumn{7}{@{}l}{\textbf{Llama-3.1-8B}} \\[1pt]
Dev. & \ci{7.4}{6.6}{8.2} & \cib{9.9}{9.0}{10.8} & \ci{-2.5}{-3.5}{-1.6} & \ci{-2.6}{-3.2}{-1.9} & \cib{-2.5}{-3.1}{-1.8} & \ci{-0.1}{-0.9}{0.7} \\
Family & \ci{8.3}{7.3}{9.1} & \cib{8.6}{7.8}{9.5} & \ci{-0.4}{-1.4}{0.6} & \cib{-1.6}{-2.3}{-0.9} & \ci{-3.7}{-4.3}{-3.1} & \ci{2.1}{1.3}{2.9} \\
$a_0$ & \ci{7.4}{6.5}{8.2} & \cib{8.5}{7.6}{9.3} & \ci{-1.1}{-1.9}{-0.4} & \cib{-2.6}{-3.2}{-1.9} & \ci{-3.6}{-4.2}{-3.0} & \ci{1.0}{0.5}{1.5} \\
\bottomrule
\end{tabular*}\endgroup

\par\vspace{6pt}
{\fontsize{8.5}{10}\selectfont\raggedright
Original RD draw. $\Grepro$ is cross-prompt headroom after subtracting development-fitted mean letter offsets from edited scores, with $a_0$ unchanged (pp). Dev. selects the static reference on adjusted development scores (primary); Family uses the original family static action, and $a_0$ uses the unmodified pass. Fixed preserves option order; Rotated applies a nonzero cyclic rotation on prompt 2. $\Delta=\Grepro(\Real)-\Grepro(\RDm)$. Brackets are 95\% paired, task-stratified item-bootstrap CIs with fitted offsets held fixed. Adjustment need not remove item-specific letter preferences. Bold marks the larger displayed Real/\RDf{} estimate in each pair.\par}
\end{minipage}
\end{table}

\paragraph{Rotated options.}
For an item with $n>1$ options, prompt 2 rotates them right by $r_i\in\{1,\dots,n-1\}$, given by a hash of the item's identifier, so option $g$ moves to position $(g+r_i)\bmod n$; the rule reads only the item's identifier and $n$.
\val{P2.selfref} MMLU-Pro items whose options refer to other options (for example ``both of the above'') are rotated like the others, with their text not edited.
Prompt 1 is the same as in the fixed-order design, and the placebo scales and static actions are unchanged.
The matching rule, applied to answer changes on the evaluation items, held on both models (\val{qwen.pm.within} and \val{llama.pm.within} of 32 programs within tolerance, median ratios \val{qwen.pm.median} and \val{llama.pm.median}).
\Cref{tab:permuted} gives the results. On the same items, rotated minus fixed, same-prompt headroom changes by \val{qwen.v14.Graw_real.diff}~\val{qwen.v14.Graw_real.diff.ci} (real) and \val{qwen.v14.Graw_rd.diff}~\val{qwen.v14.Graw_rd.diff.ci} points (\RDf) on Qwen3-4B-Base and by \val{llama.v14.Graw_real.diff}~\val{llama.v14.Graw_real.diff.ci} and \val{llama.v14.Graw_rd.diff}~\val{llama.v14.Graw_rd.diff.ci} on Llama-3.1-8B, and the unmodified model's accuracy on prompt 2 by \val{qwen.v14.acc0.diff}~\val{qwen.v14.acc0.diff.ci} and \val{llama.v14.acc0.diff}~\val{llama.v14.acc0.diff.ci} (\val{qwen.v12.acc0.p2.perm} against \val{qwen.v12.acc0.p2.fixed} and \val{llama.v12.acc0.p2.perm} against \val{llama.v12.acc0.p2.fixed}; paired task-stratified bootstrap).
With $a_0$ as the static action of both families the rotated headroom is \val{qwen.pm.a0.real} and \val{qwen.pm.a0.rd} points on Qwen3-4B-Base and \val{llama.pm.a0.real} and \val{llama.pm.a0.rd} on Llama-3.1-8B; without the \val{P2.selfref} self-referencing items, with the static actions chosen on the development items, it is \val{qwen.pm.ex.real}, \val{qwen.pm.ex.rd}, \val{llama.pm.ex.real} and \val{llama.pm.ex.rd}.

\begin{table}[htbp]
\centering
\begin{minipage}{396pt}\approvedtablefont
\caption{\textbf{Large-menu study: option order and letter offsets.}}
\label{tab:permuted}
\begingroup\fontsize{8.5}{10}\selectfont
\setlength{\tabcolsep}{2.5pt}\renewcommand{\arraystretch}{1.04}
\begin{tabular*}{\linewidth}{@{\extracolsep{\fill}}lccccc@{}}
\toprule
 & \multicolumn{2}{c}{Headroom $\Grepro$ (pp)} & Drop (pp) & \multicolumn{2}{c}{Offset-menu $\Grepro$ (pp)} \\
\cmidrule(lr){2-3}\cmidrule(lr){5-6}
Family & Fixed & Rotated & Fixed $-$ rotated & Fixed & Rotated \\
\midrule
\addlinespace[0pt]
\multicolumn{6}{@{}l}{\textbf{Qwen3-4B-Base}} \\[1pt]
Real & \ci{9.0}{8.1}{9.9} & \cib{-3.3}{-4.0}{-2.7} & \ci{12.3}{11.5}{13.2} & \num{8.8} & \num{-1.8} \\
\RDf & \cib{11.8}{10.9}{12.8} & \ci{-5.7}{-6.3}{-5.1} & \ci{17.5}{16.6}{18.4} & \num{12.2} & \num{-2.8} \\
\addlinespace[3pt]
\multicolumn{6}{@{}l}{\textbf{Llama-3.1-8B}} \\[1pt]
Real & \ci{10.1}{9.1}{11.0} & \cib{-2.5}{-3.2}{-1.8} & \ci{12.5}{11.6}{13.5} & \num{10.9} & \num{-3.0} \\
\RDf & \cib{15.6}{14.5}{16.6} & \ci{-6.1}{-6.7}{-5.5} & \ci{21.7}{20.7}{22.7} & \num{13.8} & \num{-4.4} \\
\bottomrule
\end{tabular*}\endgroup

\par\vspace{6pt}
{\fontsize{8.5}{10}\selectfont\raggedright
Original RD draw. $\Grepro$ is cross-prompt headroom over the family's development-selected static action (pp) [95\% paired, task-stratified item-bootstrap CI]. Fixed preserves option order; Rotated applies a nonzero cyclic rotation on prompt 2. Drop is fixed minus rotated headroom on the same items, with its paired interval. Letter-offset columns give point estimates for each family's development-fitted offset menu against that menu's own static action. Bold marks the larger displayed Real/\RDf{} headroom within each model and design.\par}
\end{minipage}
\end{table}

\paragraph{Ties, a positive control and an item-decoupled baseline.}
The post hoc decomposition in \cref{tab:tie-parts} classifies item--selection-direction pairs by the number of correct actions among the 33 candidates: all correct, all wrong, several correct (2--32) or uniquely correct.
All-correct pairs account for less than 0.5\%, and all-wrong pairs for about 9--14\%; each of these two categories contributes less than 0.3 percentage points in absolute value to each family's total under either rule or design.
Keeping $a_0$ whenever it is among the best actions raises every family's cross-prompt headroom and flips the real programs' rotated point estimates on both models; only Qwen's positive level has an interval excluding zero (\cref{tab:large}).
Most of this change occurs in the several-correct category.
Within that category, the rules differ only when $a_0$ is correct, by definition; the analysis did not separately estimate this subset's frequency or absolute headroom.
Under tie averaging, the several-correct category contributes \val{qwen.v15.tie.perm.tie_avg.Delta.tied_best} and \val{llama.v15.tie.perm.tie_avg.Delta.tied_best} points to rotated $\Delta$ on Qwen and Llama, and the uniquely correct category contributes \val{qwen.v15.tie.perm.tie_avg.Delta.unique_best} and \val{llama.v15.tie.perm.tie_avg.Delta.unique_best} points.
Each family is classified separately, so a category's contribution to $\Delta$ compares weighted contributions that need not come from the same items.
The positive control plants a mean-preserving effect on one real program (\val{qwen.v11.pc.name} on Qwen3-4B-Base, \val{llama.v11.pc.name} on Llama-3.1-8B: the repeat with the eighth smallest calibration churn of the 16): the log-probability of the correct option is raised by $\delta$ on a random fraction $f$ of items and lowered by $\delta f/(1-f)$ on the others, on both prompts, so that the effect follows the option's text when the options are rotated; it counts as detected when the 95\% lower bound of the increase in $\Grepro$ exceeds zero (task-stratified item bootstrap, 1,000 replicates, 200 random subsets per cell); with $\delta=1$ nat on $f=10\%$ of items it is detected in \val{L.v11.pc.power.min}--\val{L.v11.pc.power.max} of plantings in the rotated design.
The item-decoupled baseline scores each item's other prompt with the average selection weights of the other items of its task (tie sets weighted uniformly); cross-prompt oracle accuracy minus this baseline does not depend on the static action and measures stable item-action association, not computation-specific effects.
It is \val{qwen.v11.dec_real_fixed} and \val{qwen.v11.dec_rd_fixed} points for the real programs and the placebo with fixed order and \val{qwen.v11.dec_real_perm} and \val{qwen.v11.dec_rd_perm} after rotation on Qwen3-4B-Base, and \val{llama.v11.dec_real_fixed}, \val{llama.v11.dec_rd_fixed}, \val{llama.v11.dec_real_perm} and \val{llama.v11.dec_rd_perm} on Llama-3.1-8B; the differences, real minus placebo, are \val{qwen.v11.dec_Delta_fixed}~\val{qwen.v11.dec_Delta_fixed.ci} and \val{qwen.v11.dec_Delta_perm}~\val{qwen.v11.dec_Delta_perm.ci}, and \val{llama.v11.dec_Delta_fixed}~\val{llama.v11.dec_Delta_fixed.ci} and \val{llama.v11.dec_Delta_perm}~\val{llama.v11.dec_Delta_perm.ci}.

\begin{table}[htbp]
\centering
\begin{minipage}{396pt}\approvedtablefont
\caption{\textbf{Contributions by the selection prompt\textquotesingle s tie type.}}
\label{tab:tie-parts}
\begingroup\fontsize{8.5}{10}\selectfont
\setlength{\tabcolsep}{3pt}\renewcommand{\arraystretch}{1.04}
\begin{tabular*}{\linewidth}{@{\extracolsep{\fill}}llccccc@{}}
\toprule
Rule & Family & All correct & All wrong & Several & Unique & Total \\
\midrule
\multicolumn{7}{@{}l}{\textbf{Qwen3-4B-Base; shared order}} \\[1pt]
Tie average & Real & \val{qwen.v15.tie.fixed.tie_avg.real.all_correct} & \val{qwen.v15.tie.fixed.tie_avg.real.all_wrong} & \val{qwen.v15.tie.fixed.tie_avg.real.tied_best} & \val{qwen.v15.tie.fixed.tie_avg.real.unique_best} & \val{qwen.v15.tie.fixed.tie_avg.real.total} \\
 & \RDf & \val{qwen.v15.tie.fixed.tie_avg.rd.all_correct} & \val{qwen.v15.tie.fixed.tie_avg.rd.all_wrong} & \val{qwen.v15.tie.fixed.tie_avg.rd.tied_best} & \val{qwen.v15.tie.fixed.tie_avg.rd.unique_best} & \val{qwen.v15.tie.fixed.tie_avg.rd.total} \\
 & $\Delta$ & \val{qwen.v15.tie.fixed.tie_avg.Delta.all_correct} & \val{qwen.v15.tie.fixed.tie_avg.Delta.all_wrong} & \val{qwen.v15.tie.fixed.tie_avg.Delta.tied_best} & \val{qwen.v15.tie.fixed.tie_avg.Delta.unique_best} & \val{qwen.v15.tie.fixed.tie_avg.Delta.total} \\
$a_0$ kept & Real & \val{qwen.v15.tie.fixed.a0_kept.real.all_correct} & \val{qwen.v15.tie.fixed.a0_kept.real.all_wrong} & \val{qwen.v15.tie.fixed.a0_kept.real.tied_best} & \val{qwen.v15.tie.fixed.a0_kept.real.unique_best} & \val{qwen.v15.tie.fixed.a0_kept.real.total} \\
 & \RDf & \val{qwen.v15.tie.fixed.a0_kept.rd.all_correct} & \val{qwen.v15.tie.fixed.a0_kept.rd.all_wrong} & \val{qwen.v15.tie.fixed.a0_kept.rd.tied_best} & \val{qwen.v15.tie.fixed.a0_kept.rd.unique_best} & \val{qwen.v15.tie.fixed.a0_kept.rd.total} \\
 & $\Delta$ & \val{qwen.v15.tie.fixed.a0_kept.Delta.all_correct} & \val{qwen.v15.tie.fixed.a0_kept.Delta.all_wrong} & \val{qwen.v15.tie.fixed.a0_kept.Delta.tied_best} & \val{qwen.v15.tie.fixed.a0_kept.Delta.unique_best} & \val{qwen.v15.tie.fixed.a0_kept.Delta.total} \\
\addlinespace[4pt]
\multicolumn{7}{@{}l}{\textbf{Qwen3-4B-Base; rotated order}} \\[1pt]
Tie average & Real & \val{qwen.v15.tie.perm.tie_avg.real.all_correct} & \val{qwen.v15.tie.perm.tie_avg.real.all_wrong} & \val{qwen.v15.tie.perm.tie_avg.real.tied_best} & \val{qwen.v15.tie.perm.tie_avg.real.unique_best} & \val{qwen.v15.tie.perm.tie_avg.real.total} \\
 & \RDf & \val{qwen.v15.tie.perm.tie_avg.rd.all_correct} & \val{qwen.v15.tie.perm.tie_avg.rd.all_wrong} & \val{qwen.v15.tie.perm.tie_avg.rd.tied_best} & \val{qwen.v15.tie.perm.tie_avg.rd.unique_best} & \val{qwen.v15.tie.perm.tie_avg.rd.total} \\
 & $\Delta$ & \val{qwen.v15.tie.perm.tie_avg.Delta.all_correct} & \val{qwen.v15.tie.perm.tie_avg.Delta.all_wrong} & \val{qwen.v15.tie.perm.tie_avg.Delta.tied_best} & \val{qwen.v15.tie.perm.tie_avg.Delta.unique_best} & \val{qwen.v15.tie.perm.tie_avg.Delta.total} \\
$a_0$ kept & Real & \val{qwen.v15.tie.perm.a0_kept.real.all_correct} & \val{qwen.v15.tie.perm.a0_kept.real.all_wrong} & \val{qwen.v15.tie.perm.a0_kept.real.tied_best} & \val{qwen.v15.tie.perm.a0_kept.real.unique_best} & \val{qwen.v15.tie.perm.a0_kept.real.total} \\
 & \RDf & \val{qwen.v15.tie.perm.a0_kept.rd.all_correct} & \val{qwen.v15.tie.perm.a0_kept.rd.all_wrong} & \val{qwen.v15.tie.perm.a0_kept.rd.tied_best} & \val{qwen.v15.tie.perm.a0_kept.rd.unique_best} & \val{qwen.v15.tie.perm.a0_kept.rd.total} \\
 & $\Delta$ & \val{qwen.v15.tie.perm.a0_kept.Delta.all_correct} & \val{qwen.v15.tie.perm.a0_kept.Delta.all_wrong} & \val{qwen.v15.tie.perm.a0_kept.Delta.tied_best} & \val{qwen.v15.tie.perm.a0_kept.Delta.unique_best} & \val{qwen.v15.tie.perm.a0_kept.Delta.total} \\
\addlinespace[4pt]
\multicolumn{7}{@{}l}{\textbf{Llama-3.1-8B; shared order}} \\[1pt]
Tie average & Real & \val{llama.v15.tie.fixed.tie_avg.real.all_correct} & \val{llama.v15.tie.fixed.tie_avg.real.all_wrong} & \val{llama.v15.tie.fixed.tie_avg.real.tied_best} & \val{llama.v15.tie.fixed.tie_avg.real.unique_best} & \val{llama.v15.tie.fixed.tie_avg.real.total} \\
 & \RDf & \val{llama.v15.tie.fixed.tie_avg.rd.all_correct} & \val{llama.v15.tie.fixed.tie_avg.rd.all_wrong} & \val{llama.v15.tie.fixed.tie_avg.rd.tied_best} & \val{llama.v15.tie.fixed.tie_avg.rd.unique_best} & \val{llama.v15.tie.fixed.tie_avg.rd.total} \\
 & $\Delta$ & \val{llama.v15.tie.fixed.tie_avg.Delta.all_correct} & \val{llama.v15.tie.fixed.tie_avg.Delta.all_wrong} & \val{llama.v15.tie.fixed.tie_avg.Delta.tied_best} & \val{llama.v15.tie.fixed.tie_avg.Delta.unique_best} & \val{llama.v15.tie.fixed.tie_avg.Delta.total} \\
$a_0$ kept & Real & \val{llama.v15.tie.fixed.a0_kept.real.all_correct} & \val{llama.v15.tie.fixed.a0_kept.real.all_wrong} & \val{llama.v15.tie.fixed.a0_kept.real.tied_best} & \val{llama.v15.tie.fixed.a0_kept.real.unique_best} & \val{llama.v15.tie.fixed.a0_kept.real.total} \\
 & \RDf & \val{llama.v15.tie.fixed.a0_kept.rd.all_correct} & \val{llama.v15.tie.fixed.a0_kept.rd.all_wrong} & \val{llama.v15.tie.fixed.a0_kept.rd.tied_best} & \val{llama.v15.tie.fixed.a0_kept.rd.unique_best} & \val{llama.v15.tie.fixed.a0_kept.rd.total} \\
 & $\Delta$ & \val{llama.v15.tie.fixed.a0_kept.Delta.all_correct} & \val{llama.v15.tie.fixed.a0_kept.Delta.all_wrong} & \val{llama.v15.tie.fixed.a0_kept.Delta.tied_best} & \val{llama.v15.tie.fixed.a0_kept.Delta.unique_best} & \val{llama.v15.tie.fixed.a0_kept.Delta.total} \\
\addlinespace[4pt]
\multicolumn{7}{@{}l}{\textbf{Llama-3.1-8B; rotated order}} \\[1pt]
Tie average & Real & \val{llama.v15.tie.perm.tie_avg.real.all_correct} & \val{llama.v15.tie.perm.tie_avg.real.all_wrong} & \val{llama.v15.tie.perm.tie_avg.real.tied_best} & \val{llama.v15.tie.perm.tie_avg.real.unique_best} & \val{llama.v15.tie.perm.tie_avg.real.total} \\
 & \RDf & \val{llama.v15.tie.perm.tie_avg.rd.all_correct} & \val{llama.v15.tie.perm.tie_avg.rd.all_wrong} & \val{llama.v15.tie.perm.tie_avg.rd.tied_best} & \val{llama.v15.tie.perm.tie_avg.rd.unique_best} & \val{llama.v15.tie.perm.tie_avg.rd.total} \\
 & $\Delta$ & \val{llama.v15.tie.perm.tie_avg.Delta.all_correct} & \val{llama.v15.tie.perm.tie_avg.Delta.all_wrong} & \val{llama.v15.tie.perm.tie_avg.Delta.tied_best} & \val{llama.v15.tie.perm.tie_avg.Delta.unique_best} & \val{llama.v15.tie.perm.tie_avg.Delta.total} \\
$a_0$ kept & Real & \val{llama.v15.tie.perm.a0_kept.real.all_correct} & \val{llama.v15.tie.perm.a0_kept.real.all_wrong} & \val{llama.v15.tie.perm.a0_kept.real.tied_best} & \val{llama.v15.tie.perm.a0_kept.real.unique_best} & \val{llama.v15.tie.perm.a0_kept.real.total} \\
 & \RDf & \val{llama.v15.tie.perm.a0_kept.rd.all_correct} & \val{llama.v15.tie.perm.a0_kept.rd.all_wrong} & \val{llama.v15.tie.perm.a0_kept.rd.tied_best} & \val{llama.v15.tie.perm.a0_kept.rd.unique_best} & \val{llama.v15.tie.perm.a0_kept.rd.total} \\
 & $\Delta$ & \val{llama.v15.tie.perm.a0_kept.Delta.all_correct} & \val{llama.v15.tie.perm.a0_kept.Delta.all_wrong} & \val{llama.v15.tie.perm.a0_kept.Delta.tied_best} & \val{llama.v15.tie.perm.a0_kept.Delta.unique_best} & \val{llama.v15.tie.perm.a0_kept.Delta.total} \\
\bottomrule
\end{tabular*}\endgroup
\par\vspace{6pt}
{\fontsize{8.5}{10}\selectfont\raggedright
Original RD draw. Entries are post hoc point estimates in percentage points. Categories contain 33, 0, 2--32 and 1 correct actions, respectively, among the 33 candidates on the selection prompt. Each contribution averages the category indicator times (selected accuracy minus static accuracy), over items and both directions with equal task weights. Contributions sum to total headroom before rounding. Each family uses its own categories; $\Delta$ subtracts the weighted RD contribution from the corresponding real contribution, and is not a contrast on a common item subset.\par}
\end{minipage}
\end{table}

\paragraph{A design that keeps the original order.}
The rotated design always moves the correct answer; a sensitivity analysis reweights the existing outputs so that prompt 2 keeps the original order with probability $1/n$ and otherwise uses the stored rotation. Each item thus mixes two placements of its correct answer, so its correct letter is not uniform over positions when $n>2$; the correct letters are only approximately balanced across items.
In this mixture, $\Delta$ is \val{qwen.cpu.mix.D}~\val{qwen.cpu.mix.D.ci} on Qwen3-4B-Base and \val{llama.cpu.mix.D}~\val{llama.cpu.mix.D.ci} on Llama-3.1-8B, and against the item-decoupled baseline \val{qwen.cpu.mix.decD}~\val{qwen.cpu.mix.decD.ci} and \val{llama.cpu.mix.decD}~\val{llama.cpu.mix.decD.ci}, the latter interval including zero.
Real-program headroom is negative on Qwen (\val{qwen.cpu.mix.Greal}~\val{qwen.cpu.mix.Greal.ci}) and has an interval containing zero on Llama (\val{llama.cpu.mix.Greal}~\val{llama.cpu.mix.Greal.ci}); \RDf{} is negative on both (\val{qwen.cpu.mix.Grd}~\val{qwen.cpu.mix.Grd.ci} and \val{llama.cpu.mix.Grd}~\val{llama.cpu.mix.Grd.ci}).
The letter-offset menus reach \val{qwen.cpu.mix.Sreal}~\val{qwen.cpu.mix.Sreal.ci} and \val{llama.cpu.mix.Sreal}~\val{llama.cpu.mix.Sreal.ci} points (real offsets) against their own static action, but \val{qwen.cpu.mix.decSreal}~\val{qwen.cpu.mix.decSreal.ci} and \val{llama.cpu.mix.decSreal}~\val{llama.cpu.mix.decSreal.ci} against the item-decoupled baseline, so that baseline is not a null in this design.
The mixture reweights existing outputs and adds no information beyond them; it shows that the size of the residual depends on the rotation design and the reference, not that a letter-free null has been verified.

\Cref{fig:design-reference} reports the rotation and mixture contrasts under the two references; they are different estimands.

\begin{figure}[!ht]
\centering
\begin{minipage}{396pt}
\includegraphics[width=\linewidth]{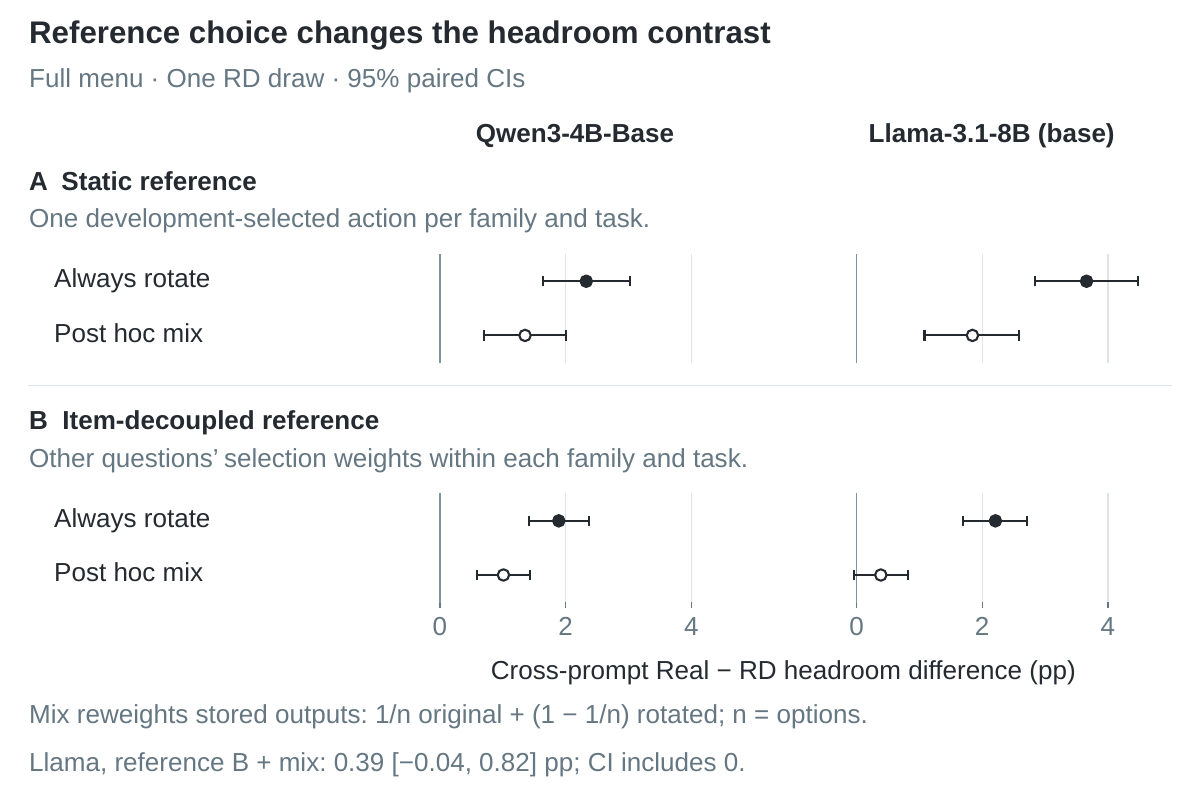}
\caption[Reference and design sensitivity]{Real-minus-RD cross-prompt headroom under two references (original RD draw). A: one static action selected on development data per family and task. B: other items' leave-one-out selection weights within the same family and task, recomputed in bootstrap. These are different estimands. Filled circles use a nonzero cyclic option rotation on prompt 2. Open circles reweight stored outputs by $1/n$ original order and $1-1/n$ rotated order, where $n$ is the option count. This post-hoc mixture adds no observations and is not a randomized or fully balanced design. Bars are 95\% paired, task-stratified item-bootstrap intervals. Both models use the same scale, the full 32-program menu plus $a_0$, tie averaging, 4,413 items and equal task weights. Llama's B-mixture interval includes zero. Zero denotes equal headrooms under the specified reference, not a validated null for computation-specific effects.}
\label{fig:design-reference}
\end{minipage}
\end{figure}

\paragraph{Calibration of the decision $\Delta>0$.}
Semi-synthetic populations resample the stored items of each task (1,471 per task) and, independently per item, swap the two families' outcomes and their static actions, so that the population $\Delta$ is exactly zero; alternatives add, on a random fraction of items, a correct answer for every real program on both prompts, a strong and unrealistic signal sized to about 1, 2 or 4 points. Each of 1,000 replicates applies the paper's decision ($\Delta>0$ if the lower end of the 95\% task-stratified bootstrap interval, 10,000 replicates, exceeds zero). With no planted effect the decision is positive in at most \val{L.cpu.pow.n0.max}\% of replicates per cell (nominal one-sided 2.5\%). A one-point advantage is detected in \val{L.cpu.pow.alt1.perm.min}--\val{L.cpu.pow.alt1.perm.max}\% of replicates in the rotated design and \val{L.cpu.pow.alt1.fixed.min}--\val{L.cpu.pow.alt1.fixed.max}\% in the fixed design, whose contrast is noisier (standard deviation of $\Delta$ \val{L.cpu.pow.sd.fixed.min}--\val{L.cpu.pow.sd.fixed.max} against \val{L.cpu.pow.sd.perm.min} points), and a two-point advantage in \val{L.cpu.pow.alt2.min}--\val{L.cpu.pow.alt2.max}\%. A normal approximation based on the original interval widths gives an 80\%-power threshold of about \val{L.v11.mde.perm.min}--\val{L.v11.mde.perm.max} points for rotated $\Delta$; the simulation has its own variance and does not validate that threshold. It assesses the decision rule on the constructed populations only.

\FloatBarrier
\section{Three-program study: data and further results}
\label{app:ladder}

\paragraph{Items and inference.}
The evaluation set has 600 held-out items per task: MMLU-Pro test items stratified by category, 40 items from each of the 15 BBH letter-choice tasks, and AQuA-RAT training items; calibration used 100 further items per task.
Few-shot examples are shared within each of 14 MMLU-Pro categories, 15 BBH tasks and one AQuA example set.
The clustered bootstrap resamples the categories and tasks, but AQuA's 600 items singly; it therefore does not capture dependence induced by AQuA's shared examples.
The static action is the action with the best mean accuracy on the selection prompt, re-chosen in every bootstrap replicate; on the pooled items it is $a_0$ for every family.
When it is instead chosen per task, the real programs' static action differs from $a_0$ on MMLU-Pro and AQuA; under that rule, too, no comparison survives the Holm correction ($\Delta$ of \val{S.pt.per_dataset.RDS.delta}, \val{S.pt.per_dataset.SHS.delta} and \val{S.pt.per_dataset.ROT.delta} points for \RDi, \SH{} and \Rot).
The bootstrap resamples items without stratification by task; two-sided bootstrap $p$-values are twice the smaller share of item-bootstrap replicates of $\Delta$ at or below and at or above zero. For \RDi, \SH{} and \Rot{} they are \val{S.pt.pooled.RDS.p}, \val{S.pt.pooled.SHS.p} and \val{S.pt.pooled.ROT.p}, so the smallest exceeds the first Holm threshold of $0.05/3$ and no comparison is rejected; the bootstrap distribution of $\Delta$ is skewed, so these $p$-values are larger than a normal approximation from the interval widths would suggest. The Bonferroni (98.33\%) intervals below are percentile intervals of the same item bootstrap.

\begin{table}[htbp]
\centering
\begin{minipage}{396pt}\approvedtablefont
\caption{\textbf{Three-program study: headroom, ratios, and paired differences.}}
\label{tab:ladder}
\begingroup\fontsize{8.5}{10}\selectfont
\setlength{\tabcolsep}{3pt}\renewcommand{\arraystretch}{1.04}
\begin{tabular*}{\linewidth}{@{\extracolsep{\fill}}lccccc@{}}
\toprule
 & \multicolumn{2}{c}{Headroom (pp)} & \multicolumn{2}{c}{Placebo / Real} & $\Delta$ (pp) \\
\cmidrule(lr){2-3}\cmidrule(lr){4-5}
Family & $\Graw$ & $\Grepro$ & $\rho_{\mathrm{same}}$ & $\rho$ & Real $-$ placebo \\
\midrule
Real & \num{16.1} & \numb{6.7} & \textemdash & \textemdash & \textemdash \\
\RDi & \num{15.3} & \num{4.6} & \ci{0.95}{0.87}{1.04} & \ci{0.69}{0.50}{0.94} & \ci{2.1}{0.4}{3.7} \\
\SH & \num{15.9} & \num{5.3} & \ci{0.99}{0.91}{1.08} & \ci{0.79}{0.61}{1.00} & \ci{1.4}{0.01}{2.9} \\
\Rot & \numb{16.7} & \num{5.4} & \ci{1.03}{0.95}{1.13} & \ci{0.81}{0.61}{1.11} & \ci{1.3}{-0.7}{2.9} \\
\bottomrule
\end{tabular*}\endgroup

\par\vspace{6pt}
{\fontsize{8.5}{10}\selectfont\raggedright
Headroom columns are point estimates (pp) relative to the family's best mean-accuracy action on the selection prompt, reselected in each bootstrap replicate. Brackets are unadjusted 95\% paired item-bootstrap intervals without task stratification. $\rho_{\mathrm{same}}$ and $\rho$ are placebo-to-real ratios of $\Graw$ and $\Grepro$; they do not estimate the share due to nonspecific effects. $\Delta=\Grepro(\Real)-\Grepro(N)$ for placebo $N$. No comparison is rejected after Holm correction. Bold marks the highest displayed estimate among the four families in each headroom column.\par}
\end{minipage}
\end{table}

\paragraph{Sensitivity.}
Resampling MMLU-Pro categories and BBH tasks as clusters gives $\Delta$ intervals of \val{S.RDS.delta.cci}, \val{S.SHS.delta.cci} and \val{S.ROT.delta.cci} points; the Bonferroni (98.33\%) intervals are \val{S.RDS.delta.ci9833}, \val{S.SHS.delta.ci9833} and \val{S.ROT.delta.ci9833} for $\Delta$ and \val{S.RDS.rcross.ci9833}, \val{S.SHS.rcross.ci9833} and \val{S.ROT.rcross.ci9833} for $\rho$.
Task-stratified item resampling leaves all three Holm-corrected decisions unchanged.
The minimum detectable $\Delta$ is a normal approximation, $2.80$ times the bootstrap standard deviation of $\Delta$ (\val{S.RDS.delta.sd}, \val{S.SHS.delta.sd} and \val{S.ROT.delta.sd} points for \RDi{}, \SH{} and \Rot{}; 80\% power, two-sided level 0.05); at the first Holm level ($0.05/3$) the factor is $3.24$, giving \val{S.mdeh.min}--\val{S.mdeh.max} points.

\begin{figure}[!ht]
\centering
\begin{minipage}{396pt}
\includegraphics[width=\linewidth]{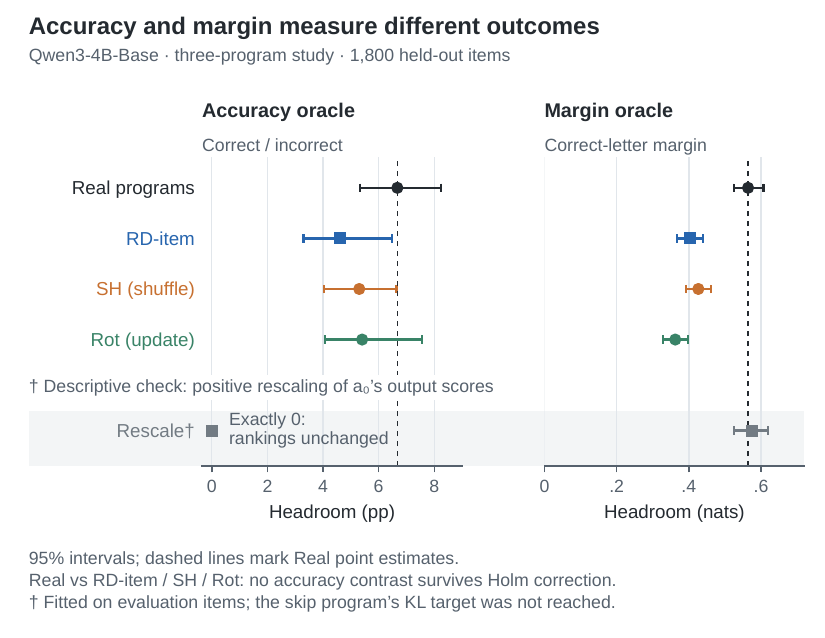}
\caption[Accuracy and margin measure different outcomes]{Cross-prompt headroom in the three-program study (Qwen3-4B-Base; 1,800 held-out items). Each column selects oracle actions and static references using its own metric: accuracy (pp) or correct-letter margin (nats). References maximize the pooled mean metric on the selection prompt and are reselected in each unstratified paired item-bootstrap replicate; intervals are 95\%. Dashed lines mark Real point estimates. No paired accuracy contrast survives Holm correction. Rot rotates hidden updates, not answer options. Rescale multiplies a$_0$'s option log-probabilities by positive factors, preserving every answer ranking: accuracy headroom is zero, while a margin oracle can select among factors. These factors were fitted on evaluation items and held fixed in bootstrap; the skip program's KL target was not reached. Rescale is a descriptive comparison, not a matched control.}
\label{fig:metric}
\end{minipage}
\end{figure}

\paragraph{Margins and output rescaling.}
Under the margin oracle the real programs exceed every placebo (\cref{fig:metric}, right panel).
Their margin gain lies mainly on rows that the unmodified model already answers correctly (\val{S.right.Real.m} against \val{S.right.plc.m.min}--\val{S.right.plc.m.max} nats per row), while on rows it answers incorrectly the placebos obtain \val{S.wrong.RDS}--\val{S.wrong.ROT} points of cross-prompt accuracy gain against the real programs' \val{S.wrong.Real}.
Output rescaling can produce margin headroom without changing predicted answers: a label-aware margin oracle gains from sharpening where the answer is right and flattening where it is wrong.
A family that only rescales the unmodified model's option log-probabilities, $\{a_0,\lambda_1\ell_0,\lambda_2\ell_0,\lambda_3\ell_0\}$ with $\ell_0$ the unmodified model's option log-probabilities and $\lambda_j>0$, changes no answer; with each factor chosen so that the KL divergence over the options matches the corresponding real program's, the factors that flatten are $(10^{-4},0.554,0.680)$, where $10^{-4}$, the lower end of the search range, is reached for the skip, whose KL cannot be matched by flattening, and the family reaches \val{S.scale.flat}~\val{S.scale.flat.ci} nats of cross-prompt margin headroom, against the real programs' \val{S.real.mrepro}~\val{S.real.mrepro.ci}.

\FloatBarrier
\section{Transfer test: data, execution and review}
\label{app:transfer}

\paragraph{Programs and problems.}
The release of \citet{westerhoff2026programs} used in this study (MIT license; commit \texttt{4164a68}) contains, for Qwen3-8B and each DART-Math difficulty, the search log of every problem: 200 search steps, each recording a layer path and its reward.
A layer path is a sequence of layer indices in which layers may be omitted or a segment repeated.
At difficulty~1, the transfer study starts from the \val{T.archive} GSM8K-sourced problems that the unmodified model answers incorrectly and whose logs contain a rewarded path.
Problems follow a fixed order defined by a salted hash of the problem identifier.
Each problem is assigned the first rewarded path in its log, with no further choice among successful candidates.
Questions and reference answers are those of the DART-Math GSM8K pool, as in the release.

\paragraph{Execution.}
Programs are applied verbatim by rerouting the model's layers (one attention copy and key-value cache slot per executed position), with the release's one-shot direct-answer prompt (one solved example, $1+1$, before the question), greedy decoding, at most 50 new tokens, bf16 weights, a fixed attention kernel, batch size~1, and the release's answer grader.
Before any rewording was evaluated, this configuration was checked against the archive: correctness agreed on \val{T.replay.identity}\% of 100 problems for the unmodified model and on \val{T.replay.rewarded}\% and \val{T.replay.unrewarded}\% of 100 archived rewarded and unrewarded problem--program pairs, respectively.
Because archived outcomes came from batched execution, every problem in the study was re-confirmed under the execution configuration described above.

\paragraph{Rewordings and review.}
Rewordings were written by an LLM (Claude Opus~5.5) in contexts that saw only the problem texts and a fixed instruction: keep every quantity, number, unit, name and condition and the question; change wording and structure; add, remove, resolve or hint at nothing; one rewording sentence by sentence, one restructured with the question last.
They were frozen before any program ran on them.
A rewording was kept only if it passed mechanical checks (identical multiset of numbers, not a copy, ends with the question) and two independent LLM reviewers, which saw neither programs nor outcomes, judged its meaning unchanged and nothing added, with one of them solving both versions to the reference answer; a problem whose two rewordings were both rejected was dropped.
Of the \val{T.rev.n} rewordings reviewed for the confirmation, \val{T.rev.rej} were rejected, most because a reviewer gave no numerical answer for the original problem.
The review is automatic and does not guarantee semantic equivalence.

\paragraph{The confirmation pool.}
Candidates excluded every problem processed in the pilot or used in the replay checks, and one near-duplicate of a retained problem; of the first \val{T.processed} of the \val{T.fresh} remaining candidates, \val{T.reconfirmed} were re-confirmed and \val{T.kept} had at least one kept rewording, so the \val{T.eligible} that met both conditions were eligible, and the first \val{T.n} in the fixed order form the pool (\val{T.groups} groups, \val{T.paraphrases} kept rewordings).
Problems were sorted by executed length and a salted hash and cut into consecutive groups of four; groups differ by at most \val{T.spread.max} executed layers, and a donor's program differs from the problem's own in \val{T.eta.pct}\% of donor pairs.
The pool, the analysis and the decision rule were fixed before any program ran on a confirmation rewording.

\begin{table}[htbp]
\centering
\begin{minipage}{396pt}\approvedtablefont
\caption{\textbf{Transfer to reworded problems: confirmation and descriptive styles.}}
\label{tab:transfer}
\begingroup\fontsize{8.5}{10}\selectfont
\setlength{\tabcolsep}{3pt}\renewcommand{\arraystretch}{1.04}
\begin{tabular*}{\linewidth}{@{\extracolsep{\fill}}lccccc@{}}
\toprule
 &  & \multicolumn{3}{c}{Accuracy (\%)} & $\Tres$ (pp) \\
\cmidrule(lr){3-5}
Rewordings & $n$ & Unmodified & Own program & Donors & Own $-$ donors \\
\midrule
All kept rewordings & \num{200} & \num{19.0} & \numb{41.5} & \num{15.5} & \ci{26.0}{20.9}{31.4} \\
Same executed length & \num{136} & \num{22.1} & \numb{43.4} & \num{16.5} & \ci{26.8}{21.3}{32.1} \\
\cmidrule(lr){1-6}
Sentence by sentence & \num{200} & \num{18.0} & \numb{49.5} & \num{15.8} & \num{33.7} \\
Restructured & \num{199} & \num{20.1} & \numb{33.2} & \num{14.9} & \num{18.3} \\
\bottomrule
\end{tabular*}\endgroup

\par\vspace{6pt}
{\fontsize{8.5}{10}\selectfont\raggedright
Post-trained Qwen3-8B; $n$ counts fresh problems conditioned on search rescue and re-confirmation. Accuracy is averaged within then across problems. Own uses the problem's search-selected program; Donors averages the other three programs in its group formed by executed length. Same executed length retains groups with four identical lengths. $\Tres$ is own minus donors (pp) [95\% group-bootstrap CI]. The overall contrast and same-length supplement were specified in advance. The style rows summarize their respective 200/199 problems; the pre-specified paired style comparison instead uses the \val{T.style.n} problems with both styles and gives \val{T.style.diff}~\val{T.style.ci} points. Rewordings passed automatic review; the comparison has no placebo. Bold marks the highest displayed accuracy within each row.\par}
\end{minipage}
\end{table}

\section{Random-direction redraws}
\label{app:redraw}

\paragraph{Design.}
Each model is evaluated with three independent random-direction draws: the original draw and two fresh draws.
Draw 1 is the original \RDf{} family used in the detailed diagnostic figures and letter-offset comparisons; the submenu tables also include draws 2--3.
Each new draw replaces the Gaussian bases, recalibrates the 32 scales on the same calibration items and reselects the placebo's static actions on the development items.
The program menu, norm profiles, prompts, option rotations, matching rules and evaluation items are fixed; the real programs and their static actions are reused.
The construction retains correlations between overlapping programs and opposite directions for the skip and repeat of the same segment.
All three draws are reported in \cref{tab:redraw-matching,tab:redraw-contrasts}; no seed was replaced.

\begin{table}[htbp]
\centering
\caption{\textbf{Matching and fixed-order headroom across random-direction draws.}
Calibration: programs within churn tolerance; evaluation matching: the full churn rule and the within-tolerance count for fixed and rotated option order.
Headroom and $C=\Grepro(\RDm)-\tfrac12\Grepro(\Real)$ are in percentage points; brackets are 95\% paired item-bootstrap intervals.
The last column gives the clustered 95\% lower bound for $C$.}
\label{tab:redraw-matching}
\begingroup
\fontsize{8.5}{10.5}\selectfont
\setlength{\tabcolsep}{4pt}
\renewcommand{\arraystretch}{1.12}
\begin{tabular}{@{}lcccccc@{}}
\toprule
Draw & Calibration & \shortstack{Matching\\fixed} & \shortstack{Matching\\rotated} & $\Grepro(\RDm)$ & $C$ & \shortstack{$C$ cluster\\lower bound} \\
\midrule
\RedrawMatchingRows
\bottomrule
\end{tabular}
\endgroup
\end{table}

\begin{table}[htbp]
\centering
\caption{\textbf{Real-minus-placebo contrasts and same-prompt headroom.}
$\Delta=\Grepro(\Real)-\Grepro(\RDm)$; all entries are percentage points.
Brackets are 95\% paired item-bootstrap intervals; the separate cluster column gives the clustered interval for rotated $\Delta$.
Same-prompt $\Graw(\RDm)$ uses the fixed-order design.}
\label{tab:redraw-contrasts}
\begingroup
\fontsize{8.5}{10.5}\selectfont
\setlength{\tabcolsep}{5pt}
\renewcommand{\arraystretch}{1.12}
\begin{tabular}{@{}lcccc@{}}
\toprule
Draw & \shortstack{$\Delta$\\fixed} & \shortstack{$\Delta$\\rotated} & \shortstack{Rotated $\Delta$\\cluster interval} & \shortstack{Same-prompt\\$\Graw(\RDm)$} \\
\midrule
\RedrawContrastsRows
\bottomrule
\end{tabular}
\endgroup
\end{table}

\begin{table}[htbp]
\centering
\begin{minipage}{396pt}\approvedtablefont
\caption{\textbf{Three-draw means and the two sources of variation.}}
\label{tab:draw-summary}
\begingroup\fontsize{8.5}{10}\selectfont
\setlength{\tabcolsep}{2pt}\renewcommand{\arraystretch}{1.10}
\begin{tabular*}{\linewidth}{@{\extracolsep{\fill}}llccc@{}}
\toprule
Model & Contrast & Mean [item interval] & Draw SD & Draw range \\
\midrule
Qwen3-4B-Base & $C$, shared & \storedci{qwen.v16.draw.C.mean} & \val{qwen.v16.draw.C.sd} & \val{qwen.v16.draw.C.range} \\
 & $\Delta$, shared & \storedci{qwen.v16.draw.Delta_fixed.mean} & \val{qwen.v16.draw.Delta_fixed.sd} & \val{qwen.v16.draw.Delta_fixed.range} \\
 & $\Delta$, rotated & \storedci{qwen.v16.draw.Delta_rotated.mean} & \val{qwen.v16.draw.Delta_rotated.sd} & \val{qwen.v16.draw.Delta_rotated.range} \\
\addlinespace[3pt]
Llama-3.1-8B & $C$, shared & \storedci{llama.v16.draw.C.mean} & \val{llama.v16.draw.C.sd} & \val{llama.v16.draw.C.range} \\
 & $\Delta$, shared & \storedci{llama.v16.draw.Delta_fixed.mean} & \val{llama.v16.draw.Delta_fixed.sd} & \val{llama.v16.draw.Delta_fixed.range} \\
 & $\Delta$, rotated & \storedci{llama.v16.draw.Delta_rotated.mean} & \val{llama.v16.draw.Delta_rotated.sd} & \val{llama.v16.draw.Delta_rotated.range} \\
\addlinespace[3pt]
\bottomrule
\end{tabular*}\endgroup
\par\vspace{6pt}
{\fontsize{8.5}{10}\selectfont\raggedright
Full menus, tie averaging; all entries in percentage points. The post hoc mean interval uses a common task-stratified item bootstrap across the three draws. It conditions on their realized bases, calibration and static choices, and does not include uncertainty over new draws. The sample SD and range describe the three draw-level point estimates; they are not confidence intervals. Real outputs are shared across draws, making the SDs of $C$ and fixed-order $\Delta$ identical.\par}
\end{minipage}
\end{table}

\FloatBarrier
\paragraph{Results and uncertainty.}
\val{redraw@findings}
$C$ ranges from \val{redraw@qwen.C1.range} points for Qwen and from \val{redraw@llama.C1.range} for Llama across the three draws, with sample standard deviations of \val{redraw@qwen.C1.sd} and \val{redraw@llama.C1.sd} points.
Rotated-order $\Delta$ ranges from \val{redraw@qwen.D_perm.range} and from \val{redraw@llama.D_perm.range} points, respectively, with standard deviations of \val{redraw@qwen.D_perm.sd} and \val{redraw@llama.D_perm.sd} points.
These ranges and standard deviations describe the three draws; they are not intervals for the random-direction distribution.
Bootstrap intervals quantify evaluation-item uncertainty within each draw.
For the three-draw means in \cref{tab:draw-summary}, the analysis averages each item's contributions across draws before applying one common task-stratified item bootstrap.
These post hoc intervals condition on the three realized controls and include their shared-item dependence, but not uncertainty over fresh control draws.
Draw SDs and ranges are reported separately; three draws provide limited information about the random-direction distribution.
A rough sensitivity check treats the three directions as a random sample.
It adds the between-draw variance of the mean (SD$^2/3$) to the item-bootstrap variance of the three-draw mean, estimated from its 95\% interval half-width, and uses a $t$ quantile with two degrees of freedom.
The resulting intervals are \val{qwen.v16.draw.C.dir.ci} and \val{llama.v16.draw.C.dir.ci} for $C$, \val{qwen.v16.draw.Delta_fixed.dir.ci} and \val{llama.v16.draw.Delta_fixed.dir.ci} for fixed-order $\Delta$, and \val{qwen.v16.draw.Delta_rotated.dir.ci} and \val{llama.v16.draw.Delta_rotated.dir.ci} for rotated $\Delta$ (Qwen, Llama).
Thus $C>0$ and Llama's two orderings survive this check, whereas Qwen's fixed-order ordering and rotated difference are established for the realized draws but not for the construction.
Cluster resampling uses MMLU-Pro categories, BBH tasks and individual ARC items.
The models share evaluation items, and each model reuses its real-family outputs across draws.
A positive rotated $\Delta$ is a relative comparison between families and does not establish positive transferable headroom or a computation-specific effect.

Fixed-order cross-prompt placebo headroom ranges from \val{redraw@qwen.G_rd.range} points for Qwen and from \val{redraw@llama.G_rd.range} for Llama, against real-program headroom of \val{redraw@qwen.d1.G_real} and \val{redraw@llama.d1.G_real} points.
Same-prompt placebo headroom ranges from \val{redraw@qwen.Graw_rd.range} and from \val{redraw@llama.Graw_rd.range} points, respectively; the real programs give \val{redraw@qwen.d1.Graw_real} and \val{redraw@llama.d1.Graw_real} points.
The largest fitted norm multiplier in draws 1--3 is \val{redraw@qwen.d1.scale.max}, \val{redraw@qwen.d2.scale.max} and \val{redraw@qwen.d3.scale.max} for Qwen, and \val{redraw@llama.d1.scale.max}, \val{redraw@llama.d2.scale.max} and \val{redraw@llama.d3.scale.max} for Llama.
Thus large scales remain part of the construction, even though their size varies across draws.

\section{Overview example: selection and coordinates}
\label{app:figure-notes}

\Cref{fig:overview} illustrates a post-hoc selected Qwen3-4B-Base program pair: repeating layers 11--14 and its original paired \RDf{} control.
Of the full 4,413-item evaluation set, all 40 items for which $a_0$ is wrong on both prompts and both edits are correct on Prompt 1 are shown, with no additional margin-based display filter.
Their Prompt-2 outcome frequencies describe this conditional subset, rather than the full oracle menu.
The aggregate headroom analysis instead averages accuracy ties, items and both selection directions against each family's development-selected static reference.

Each coordinate is the correct option's log-probability minus that of the strongest incorrect option on one prompt; all tiles use the same linear axes.
Arrows connect an item's unmodified and edited outputs, not hidden states.
The inset enlarges the lunar-eclipse item already shown in the grid and adds no observation.

\FloatBarrier
\section{Reproducibility}
\label{app:repro}

\paragraph{Code and data.}
The code, item sets, calibration results and per-item outputs (letter log-probabilities of every candidate, and KL divergences), and for the transfer test the rewordings, review records and per-generation outputs, will be released.
\ifanonymous The review supplement provides\else The release will also include\fi{} the internally recorded analysis plans and machine-readable covariance, across-draw and submenu results.

\paragraph{Models.}
Qwen/Qwen3-4B-Base (revision \texttt{906bfd4b}), meta-llama/Llama-3.1-8B (revision \texttt{d04e592b}); for the transfer test Qwen/Qwen3-8B (revision \texttt{b968826d}).

\paragraph{Environment and compute.}
Python 3.12, PyTorch 2.11 (CUDA 12.8), Transformers 5.17, one NVIDIA RTX 5090.
The project used approximately 150 GPU hours, mostly on this GPU.

\end{document}